\documentclass{article}

\usepackage{microtype}
\usepackage{graphicx}
\usepackage{booktabs} 

\usepackage[margin=1in]{geometry}
\usepackage{times}
\usepackage{hyperref}

\usepackage{graphicx}
\usepackage{subcaption}
\usepackage{graphicx}
\usepackage{color,soul}

\usepackage{float}

\usepackage[utf8]{inputenc} 
\usepackage[T1]{fontenc}    
\usepackage{url}            
\usepackage{booktabs}       
\usepackage{amsfonts}       
\usepackage{nicefrac}       
\usepackage{microtype}      
\usepackage{tcolorbox}
\definecolor{darkred}{RGB}{150,0,0}
\definecolor{darkgreen}{RGB}{0,150,0}
\definecolor{darkblue}{RGB}{0,0,150}
\hypersetup{colorlinks=true, linkcolor=red, citecolor=darkgreen, urlcolor=darkblue}

\usepackage{amssymb}
\usepackage{amsmath}
\usepackage{amsmath,amsfonts,amssymb,amsthm}
\usepackage{mathtools}
\usepackage{commath}
\usepackage{flexisym}
\usepackage{csquotes}
\usepackage[english]{babel}
\usepackage[utf8]{inputenc}
\usepackage{algorithm}
\usepackage[noend]{algpseudocode}

\usepackage{color}

\usepackage{amsmath}
\DeclareMathOperator*{\argmax}{arg\,max}

\newtheorem{lemma}{Lemma}
\newtheorem{theorem}{Theorem}

\newtheorem{definition}{Definition}

\newcommand*{\affaddr}[1]{#1} 

\newcommand*{\email}[1]{\texttt{#1}}

\newcommand{\Des}{\Dc_\epsilon^{\text{s}}}
\newcommand{\Dts}{\Dc_t^{\text{s}}}
\newcommand{\Ds}{\Dc_0^{\text{s}}}

\newcommand{\Dst}{\Dc^{\text{s}}}

\newcommand{\SGPUCB}{SGP-UCB}
\newcommand{\Dw}{\Dc^w}
\newcommand{\simiid}{\stackrel{\rm iid}{\sim}}

\newcommand{\lamin}{\la_{\rm \min}}
\newcommand{\lamax}{\la_{\rm \max}}
\newcommand{\la}{\lambda}

\newcommand{\nn}{\nonumber}

\newcommand{\bal}{\begin{align}}
\newcommand{\eal}{\end{align}}

\DeclarePairedDelimiterX{\inp}[2]{\langle}{\rangle}{#1, #2}

\newcommand{\X}{\mathbf{X}}
\newcommand{\K}{\mathbf{K}}
\newcommand{\I}{\mathbf{I}}
\newcommand{\Ab}{\mathbf{A}}
\newcommand{\Y}{\mathbf{Y}}

\newcommand{\x}{\mathbf{x}}

\newcommand{\vb}{\mathbf{v}}

\newcommand{\y}{\mathbf{y}}

\newcommand{\z}{\mathbf{z}}

\newcommand{\f}{\mathbf{f}}
\newcommand{\kb}{\mathbf{k}}

\newcommand{\Dc}{\mathcal{D}}

\newcommand{\Oc}{\mathcal{O}}

\newcommand{\beq}{\begin{equation}}
\newcommand{\eeq}{\end{equation}}
\newcommand{\bea}{\begin{align}}
\newcommand{\eea}{\end{align}}

\newcommand{\E}{\mathbb{E}}

\title{Regret Bounds for Safe Gaussian Process Bandit Optimization}

\author{
Sanae Amani, Mahnoosh Alizadeh, and  Christos Thrampoulidis\\
\affaddr{University of California, Santa Barbara}\\
\email{\{samanigeshnigani,alizadeh,cthrampo\}@ucsb.edu}\\
}

\usepackage[parfill]{parskip}

\begin{document}

\date{}
\maketitle
\begin{abstract}
Many applications require a learner to make sequential decisions given uncertainty regarding both the system's payoff function and safety constraints. In safety-critical systems, it is paramount that the learner's actions do not violate the  safety constraints at any stage of the learning process. 
In this paper, we study a stochastic bandit optimization problem where the unknown payoff and constraint functions are sampled from Gaussian Processes (GPs) first considered in \cite{srinivasgaussian}. We develop a safe variant of GP-UCB called \SGPUCB, with necessary modifications to respect safety constraints at every round. The algorithm has two distinct phases. The first phase seeks to estimate the set of safe actions in the decision set, while the second phase follows the GP-UCB decision rule. Our main contribution is to derive the first sub-linear regret bounds for this problem. We numerically compare \SGPUCB~against existing safe Bayesian GP optimization algorithms.
\end{abstract}
\section{Introduction}
Stochastic bandit optimization has received significant attention in applications where a learner must repeatedly deal with an unknown random environment and observations are costly to obtain. At each round, the learner chooses an action $\x$ and observes a noise-perturbed version of an otherwise unknown reward function $f(\x)$. 
The goal is to minimize the so-called cumulative pseudo-regret, i.e., the difference between the expected $T$-period reward generated
by the  algorithm and the optimal expected reward if $f$ was known to the learner. The most well-studied case is when the unknown function $f$ comes from a finite dimensional linear model, with regret bounds provided by \cite{Dani08stochasticlinear,abbasi2011improved,rusmevichientong2010linearly} for Upper Confidence Bound (UCB) based algorithms. In a more general setting, the expected reward  is a sample from a Gaussian Process (GP), with regret bounds first provided by \cite{srinivasgaussian}. GPs are a popular choice for modelling reward function in Bayesian optimization methods as well as experimental design with applications in medical trials and robotics, e.g.,  \cite{berkenkamp2016safe,akametalu2014reachability,ostafew2016robust,berkenkamp2016bayesian}. In a closely related line of work, \cite{srinivasgaussian,valko2013finite,chowdhury2017kernelized} proposed kernelized UCB algorithms for settings where the reward is an unknown {\it arbitrary} function in a reproducing kernel Hilbert space (RKHS), and provided regret bounds through a frequentist analysis. A variant of this problem considers an environment that is also subject to a number of  \emph{unknown} safety constraints. The application of stochastic bandit optimization in safety-critical systems requires the learner to select actions that satisfy these safety constraints at each round, in spite of uncertainty regarding the safety requirements.

In this paper, we consider a stochastic bandit optimization problem where both the reward function $f$ and the constraint function $g$  are samples from Gaussian Processes. We require that the learner's chosen actions respect safety constraints at every round in spite of uncertainty about safe actions. This setting was first studied in \cite{Krause} in the specific case of a single safety constraint of the form $f(\x) \geq h$ and later in  \cite{sui2018stagewise}, in the more general case  of $g(\x) \geq h$ as adopted in our paper. In this case, the learner hopes to overcome the two-fold challenge of keeping the cumulative regret as small as possible while ensuring that selected actions respect the safety constraints at each round of the algorithm.  We present \SGPUCB, which is a safety-constrained variant of GP-UCB proposed by \cite{srinivasgaussian}. To ensure constraint satisfaction, \SGPUCB~restricts the learner to choose actions from a conservative inner-approximation of the safe decision set that is known to satisfy safety constraints with high probability given the algorithm's history. The cumulative regret bound of our proposed algorithm (given in Section \ref{sec:regret} as our main theoretical result) implies that \SGPUCB~is a {\it no-regret} algorithm. This is the main difference of our results compared to the algorithms studied in \cite{Krause, sui2018stagewise} that only come with convergence-but, no regret- guarantees. Throughout the paper, we discuss in detail the differentiating features of our algorithm from existing ones.



\noindent{\bf Notation.} We use lower-case letters for scalars, lower-case bold letters for vectors, and upper-case bold letters for matrices. The Euclidean norm of a vector $\x$ is denoted by $\|\x\|_2$. We denote the transpose of any column vector $\x$ by $\x^{T}$. Let $\Ab$ be a positive definite $d\times d$ matrix and $\vb \in \mathbb R^d$. The weighted 2-norm of $\vb$ with respect to $A$ is defined by $\|\vb\|_\Ab = \sqrt{\vb^T \Ab \vb}$. We denote the minimum and maximum eigenvalue of $A$ by $\lamin(\Ab)$ and $\lamax(\Ab)$. The maximum of two numbers $\alpha, \beta$ is denoted $\alpha\vee\beta$. For a positive integer $n$, $[n]$ denotes the set $\{1,2,\ldots,n\}$. 

\subsection{Problem Statement}\label{sec:problemstate}

\vspace{-0.1cm}
 The learner is given a finite
decision set $\Dc_0 \subset \mathbb R^d$.  At each round $t$, she chooses an action $\x_t\in \Dc_0$ and observes a noise-perturbed value of an unknown reward function $f: \Dc_0 \rightarrow \mathbb{R}$, i.e. $y_t := f(\x_t)+ \eta_t$. At every round, the learner must ensure that the chosen action $\x_t$ satisfies the following safety constraint:
 \begin{equation} \label{eq:constraint}
     g(\x_t) \geq h,
\end{equation}
where $g: \Dc_0 \rightarrow \mathbb{R}$ is an unknown function and $h$ is a known constant.\footnote{Our results can be simply extended to the settings with several safety constraints, i.e., set of $g_i$'s and $h_i$'s, however, for the sake of brevity we focus on one constraint function.} We define the safe set from which the learner is allowed to take action as:
\begin{align}\label{eq:Ds}
    \Ds:=\{\x\in \Dc_0\,:\, g(\x) \geq h\}.
\end{align}
Since $g$ is unknown, the learner cannot identify $\Ds$. As such, the best she can do is to choose actions $\x_t$ that are in $\Ds$ {\it with high probability}.
We assume that at every round, the learner also receives noise-perturbed feedback on the safety constraint, i.e. $z_t := g(\x_t)+ n_t$.  


\noindent{\bf Goal.} Since our knowledge of $g$ comes from noisy observations, we are not able to fully identify the true safe set $\Ds$ and infer $g(\x)$ exactly, but only up to some statistical confidence $g(\x)\pm \epsilon$ for some $\epsilon>0$. Hence, 
we consider the optimal action through an $\epsilon$-reachable
safe set for some $\epsilon>0$: 
\begin{equation}\label{eq:epsilon_reachable}
    \Des := \{\x \in \Dc_0: g(\x) \geq h+\epsilon\},
\end{equation} 
as our benchmark. A natural performance metric in this context is \textit{cumulative pseudo-regret} \cite{audibert2009exploration} over the course of $T$ rounds, which is defined by $R_T = \sum_{t=1} ^T  f(\x^*_{\epsilon}) - f(\x_t),$ 
where $\x^*_{\epsilon}$ is the optimal \emph{safe} action that maximizes the reward in expectation over the $\Des$, i.e., $\x^*_{\epsilon} \in \argmax_{\x\in \Des}f(\x).$ For the rest of this paper, we simply use regret to refer to the pseudo-regret $R_T$ and drop the subscript $\epsilon$ from $\x_\epsilon^*$.

A desirable asymptotic property of a learning algorithm is that $R_T/T\rightarrow 0$ as fast as possible as $T$ grows, especially when actions are costly. Algorithms with this property are called \emph{no-regret}. Thus, the goal of the learner is to follow a no-regret algorithm while ensuring all actions she chooses are safe with high probability.

\noindent{\bf Regularity Assumptions.} The above specified goal cannot be achieved unless certain assumptions are made on $f$ and $g$. In what follows, we assume that these functions have a certain degree of smoothness. This assumption allows us to  model the reward function $f$ and the constraint function $g$ as a sample from a Gaussian Process (GP) \cite{williams2006gaussian}. We now present necessary standard terminology and notations on GPs. A $GP(\mu(\x),k(\x,\x'))$ is a probability distribution across a class of smooth functions, which is parameterized by a kernel function $k(\x,\x')$ that characterizes the smoothness of the function. The Bayesian algorithm we analyze uses $GP(0,k_f(\x,\x'))$ and $GP(0,k_g(\x,\x'))$ as prior distributions over $f$ and $g$, respectively, where $k_f$ and $k_g$ are positive semi-definite kernel functions. Moreover, we assume bounded variance $k_f(\x,\x)\leq 1$ and $k_g(\x,\x)\leq 1$. For a noisy sample $\y_t = [y_1,\ldots,y_t]^T$, with i.i.d Gaussian noise $\eta_t \sim \mathcal{N}(0,\sigma^2)$ the posterior over $f$ is also a GP with the mean $\mu_{f,t}(\x)$ and variance $\sigma_{f,t}^2(\x)$:
\begin{align}
    \mu_{f,t}(\x) &= \kb_{f,t}(\x)^T(\K_{f,t}+\sigma^2\I)^{-1}\y_t,\\
    {\sigma^{2}_{f,t}}(\x)&=k_{f,t}(\x,\x),
\end{align}
where $k_{f,t}(\x,\x')=k_{f}(\x,\x')-\kb_{f,t}(\x)^T(\K_{f,t}+\sigma^2\I)^{-1}\kb_{f,t}(\x')$, $\kb_{f,t}(\x) = [k_f(\x_1,\x),\ldots,k_f(\x_t,\x)]^T$ and $\K_{f,t}$ is the positive definite kernel matrix $[k_f(\x,\x')]_{\x,\x'\in\{\x_1,\ldots,\x_t\}}$. Associated with $g$, the mean $\mu_{g,t}(\x)$ and variance ${\sigma_{g,t}^{2}}(\x)$ are defined similarly.

\subsection{\bf Related work}

As mentioned in the introduction, the most closely related works to this paper are \cite{Krause,sui2018stagewise}. With the objective function denoted by $f(\x)$, \cite{Krause} adopts a single constraint of the form $f(\x) \geq h$, whereas \cite{sui2018stagewise,berkenkamp2016bayesian} consider the more general constraint set $g_i(\x) \geq h_i, i\in[m]$. As is the case in our paper, the objective and constraint are modeled by Gaussian Processes. For algorithmic design purposes, \cite{Krause,sui2018stagewise} further assume Lipschitzness on reward and constraint functions. These assumptions are not required in our framework. Moreover, both of the aforementioned papers seek to identify a safe decision with the highest possible reward given a limited number of trials; i.e., their goal is to provide  {\it best-arm identification}  with convergence guarantees. Instead,  our paper focuses on a long-term performance characterized through cumulative regret bounds. A more detailed comparison to algorithms and guarantees of \cite{Krause,sui2018stagewise} is given in Section \ref{sec:counterexample}.


There exist other contexts where safety constraints have been applied to stochastic bandit optimization frameworks.
To name a few, the recent work of \cite{usmanova2019safe}  studies a safe  variant of the Frank-Wolfe algorithm to solve a smooth optimization problem with unknown convex objective and unknown \textit{linear} constraints that are accessed by the learner via stochastic zeroth-order feedback.  The analysis aims at providing sample complexity results and convergence guarantees, whereas we aim to provide regret bounds. In contrast to our setting, the paper \cite{usmanova2019safe} requires multiple measurements of the constraint at each round of the algorithm. Other closely related works of \cite{amani2019linear,amani2020generalized} study the problem of safe linear and generalized linear stochastic bandit where the constraint and loss functions depend linearly (directly or via a link function) on an unknown parameter. In fact, our algorithm can be seen as an extension of Safe-LUCB proposed by \cite{amani2019linear} to safe GPs. Specifically, in Section \ref{sec:finite}, we show that our algorithm and guarantees are similar to those in \cite{amani2019linear} for linear kernels. While \cite{amani2019linear} studies a frequentist setting, our results hold for a rich class of kernels beyond linear kernel. 


In a broader sense, the problem of safe learning has received significant attention in reinforcement learning and controls. For example, \cite{berkenkamp2017safe} combines classical reinforcement learning with stability requirements by applying a Gaussian process prior to learn about system dynamics and shows improvement in both control performance and safe region expansion. Another notable work is  \cite{schreiter2015safe}, which presents an active learning framework that uses Gaussian Processes to learn the safe decision set. In \cite{turchetta2016safe}, the authors address the problem of safely exploring finite Markov decision processes (MDP), where state-action pairs are associated with safety features that  are modeled by Gaussian processes and must lie above a threshold.  Also in the MDP setting, \cite{moldovan2012safe} proposes an algorithm that allows safe exploration in order to avoid fatal absorbing states  that must never be visited during the exploration process. By considering constrained MDPs that are augmented with a set of auxiliary cost functions and replacing them with surrogates that are easy to estimate,  \cite{achiam2017constrained} proposes a policy search algorithm for constrained reinforcement learning with guarantees for near constraint satisfaction at each iteration. Furthermore, \cite{wachi2018safe} presents a reinforcement learning approach to explore and optimize a safety-constrained MDP where the safety values of states are modeled by GPs.
From a control theoretic point of view, the recent work \cite{liu2019robust} studies an algorithmic framework for safe exploration in model-based control which comes with convergence guarantees, but no regret bound. Other notable work in this area include \cite{gillulay2011guaranteed} that combines reachability analysis and machine learning for autonomously
learning the dynamics of a target vehicle and \cite{aswani2013provably} that designs a learning-based MPC scheme that provides
deterministic guarantees on robustness when the underlying system model is linear and has a known level of uncertainty.



\vspace{-0.1cm}
\section{A Safe GP-UCB Algorithm}\label{sec:algo}
\vspace{-0.2cm}
We start with a description of \SGPUCB, which is summarized in Algorithm \ref{algo:Safe GP-UCB}. Similar to a number of previous works (e.g., \cite{sui2018stagewise,amani2019linear}), \SGPUCB~proceeds in two phases to balance the goal of expanding the safe set and controlling the regret. Prior to designing the decision rule, the algorithm requires a proper expansion of $\Ds$. Hence, in the first phase, it takes actions at random from a given safe seed set $\Dc^w$ until the safe set has sufficiently expanded (discussion on other suitable sampling strategies in the first phase is provided in Appendix \ref{sec:morediscuss}). In the second phase,  the algorithm exploits  GP properties to make predictions of $f$ from past noisy observations $y_t$. It then follows the {\it Upper Confidence Bound} (UCB) machinery to select the action. In the absence of constraint \eqref{eq:constraint}, UCB-based algorithms select action $\x_t$ such that $f(\x_t)$ is a high probability upper bound on $f(\x^*)$. Specifically, \SGPUCB~constructs appropriate confidence intervals $Q_{f,t}(\x)$ of $f$ for $\x \in \Dc_0$ (see \eqref{eq:confidence_interval1}). However, the safety constraint \eqref{eq:constraint} requires the algorithm to have a more delicate sampling rule as follows.

The algorithm exploits the noisy constraint observations $z_t$ to similarly establish confidence intervals $Q_{g,t}(\x)$ for the unknown constraint function such that $g(\x) \in Q_{g,t}(\x)$ for $\x \in \Dc_0$.
These confidence intervals allow us to design an inner approximation $\Dts$ of the safe set (see \eqref{eq:safeset}). The chosen actions belong to $\Dts$ which guarantees that the safety constraint $\eqref{eq:constraint}$ is met with high probability. In sections \ref{sec:firstphase} and \ref{sec:secondphase}, we explain the first and second phases of the algorithm in detail.

 
 



\subsection{First Phase: Exploration phase}\label{sec:firstphase}
The exploration phase aims to reach a sufficiently expanded safe subset of $\Dc_0$. The stopping criterion for this phase is to reach an approximate safe set within which $\x^*$ lies with high probability.
The algorithm starts exploring by choosing actions from $\Dc^w$ at random (see Appendix \ref{sec:morediscuss} for discussion on alternative suitable sampling rules in the first phase). After $T'$ rounds of exploration, where $T'$ is passed as an input to the algorithm, \SGPUCB~exploits the collected observations $z_t, t\in[T']$ to obtain a reasonable estimate of the unknown function $g$ and consequently to establish an expanded safe set which contains $\x^*$ with high probability.

\begin{algorithm}
 \caption{\SGPUCB ($\delta$, $\epsilon$, $\Dc_0$, $\Dc^w$, $\la_-(\tilde \la_-)$, $T'$, $T$)} \label{algo:Safe GP-UCB} 
\begin{algorithmic}[1]
\State \: \: \textbf{Pure exploration phase:}
   \State \: \: \textbf{for} $t=1\ldots,T'$
   \State \: \: \: \: Randomly choose $\x_{t} \in \Dw$ and observe $y_t$ and $z_t$.
   \State \: \: \textbf{end for}
\State \: \: \textbf{Safe exploration-exploitation phase:}
   \State \textbf \: \:\textbf{for} $t=T'+1\ldots,T$
   \State \: \: \: \: {Compute $\ell_{f,t}$, $u_{f,t}$, $\ell_{g,t}$, and $u_{g,t}$ using \eqref{eq:lower} and \eqref{eq:upper} and $\beta_t$ specified in Theorem \ref{thm:confidence_interval}.}
\State \: \: \: \: {Create $\Dts$ as in \eqref{eq:safeset}.}
\State \: \: \: \: {Choose $\x_t = \argmax_{\x \in \Dts}u_{f,t}(\x)$ and observe $y_t$, $z_t$.}
\State \: \: \textbf{end for}


\end{algorithmic}
\end{algorithm}

\subsection{Second Phase: Exploration-Exploitation phase}\label{sec:secondphase}
In the second phase, the algorithm follows an approach similar to GP-UCB
\cite{srinivasgaussian} in order to balance exploration and exploitation and guarantee the no-regret property. At rounds $t=T'+1, \ldots , T$, \SGPUCB~ uses previous observations to estimate $\Ds$ and predict $f$. It creates the following confidence interval for $f(\x)$:
\begin{align}
    Q_{f,t}(\x)&:=[\ell_{f,t}(\x),u_{f,t}(\x)],\label{eq:confidence_interval1}
\end{align}
where, 
\begin{align}
    \ell_{f,t}(\x) &= \mu_{f,t-1}(\x)-\beta_t^{1/2}\sigma_{f,t-1}(\x),\label{eq:lower}\\
    u_{f,t}(\x) &= \mu_{f,t-1}(\x)+\beta_t^{1/2}\sigma_{f,t-1}(\x)\label{eq:upper}.
\end{align}
Confidence intervals $Q_{g,t}(\x)$ corresponding to $g(\x)$ are defined in a similar way. We choose $\beta_t$ according to Theorem \ref{thm:confidence_interval} to guarantee $f(\x) \in Q_{f,t}(\x)$ and $g(\x) \in Q_{g,t}(\x)$ for all $\x \in \Dc_0$ and $t>0$ with high probability.
\begin{theorem}[Confidence Intervals, \cite{srinivasgaussian}]\label{thm:confidence_interval} Pick $\delta \in (0,1)$ and set $\beta_t = 2\log((2)|\Dc_0|t^2\pi^2/6\delta)$, then:
\begin{align}
    f(\x) \in Q_{f,t}(\x) ~,~ g(\x) \in Q_{g,t}(\x),~\forall \x \in \Dc_0, t>0,\nn
\end{align}
with probability at least $1-\delta$.
\end{theorem}
Using the above defined confidence intervals $Q_{f,t}(\x)$ and $Q_{g,t}(\x)$, the algorithm is able to act conservatively to ensure that safety constraint \eqref{eq:constraint} is satisfied. 
Specifically, at the beginning of each round $t=T'+1,\ldots,T$, \SGPUCB~forms the following so-called {\it safe decision sets} based on the mentioned confidence bounds:
\begin{align}\label{eq:safeset}
\Dts :=\{\x \in \Dc_0: \ell_{g,t}(\x)\geq h\}.
\end{align}
Recall that $g(\x)\geq \ell_{g,t}(\x)$ for all $t>0$ with high probability. Therefore, $\Dts$ is guaranteed to be a set of safe actions with the same probability. After creating safe decision sets in the second phase, the algorithm follows a similar decision rule as in GP-UCB algorithm in \cite{srinivasgaussian}. Specifically, $\x_t$ is chosen such that:
\begin{align}\label{eq:decisionrule}
    \x_t = \argmax_{\x \in \Dts}u_{f,t}(\x).
\end{align}

\section{Regret Analysis of \SGPUCB}\label{sec:regret}
Consider the following decomposition on the cumulative regret:
\begin{align}
    R_T = \underbrace{\sum_{t=1}^{T'}r_t}_{\rm Term~I}\,+\,\underbrace{\sum_{t=T'+1}^{T}r_t}_{\rm Term~II}, \label{eq:Terms}
\end{align} 
where $r_t = f(\x^*)-f(\x_t)$ is the instantaneous regret at round $t$. 

\noindent{\bf Bounding Term II.}  The main challenge in the analysis of \SGPUCB ~compared to the classical GP-UCB is that $\x^*$ may not lie within the estimated safe set $\Dts$ at all rounds of the algorithm's second phase if $T'$ is not properly chosen. 

In the following sections, we show how  $T'$ is appropriately chosen such   that $\x^* \in \Dts$ with high probability for all $t\geq T'+1$. Having said that, we bound the second term of \eqref{eq:Terms} using the standard regret analysis which appears in \cite{srinivasgaussian}. Specifically, the bound depends on the so-called {\it information gain} $\gamma_t$  which quantifies how fast $f$ can be learned in an information theoretic sense.
Concretely, $\gamma_t := \max_{|A|\leq t} I(f;\y_{A})$,
is the maximal mutual information that can be obtained about the GP prior from $t$ samples.
Information gain is a problem dependent quantity: its value depends on the given decision set $\Dc_0$ and kernel function $k_f$. For any finite $\Dc_0$, it holds \cite{srinivasgaussian}:
\begin{align}
    \gamma_t \leq |\Dc_0|\log\left(1+\sigma^{-2}t|\Dc_0|\max_{\x\in \Dc_0}k_f(\x,\x)\right).
\end{align}
While $\gamma_t$ is generally bounded by $\Oc(|\Dc_0|\log t|\Dc_0|)$, it has a sublinear dependence on $|\Dc_0|$ for commonly used kernels (e.g. Gaussian kernel).

\noindent{\bf Bounding Term I.} Since for the first $T'$ rounds actions are selected at random, the bound on Term I is linear in $T'$. In other words, the upper bound on the first term is of the form $BT'$, where $B:= C\sqrt{2\ell d} ~\text{diam} (\Dc_0)/\delta$ for some $C>0$ if $k_f$ is an RBF kernel with parameter $\ell$, otherwise $B :=  2\sqrt{2\log(2|\Dc_0|)}/\delta$ such that (see Lemma \ref{lemm:B} in Appendix \ref{sec:prooffinite1} for details):
\begin{align}
    \Pr\left(\max_{\x,\y \in \Dc_0}|f(\x)-f(\y)|<B\right)\geq 1-\delta.
\end{align}
Next, we need to find the value of $T'$ such that with high probability $\x^* \in \Dts$ for all $t\geq T'+1$. The following lemma, proved in Appendix \ref{sec:proofxstarindts}, establishes a sufficient condition for $\x^* \in \Dts$ which is more convenient to work with.

\begin{lemma}[$\x^*\in \Dts$]\label{lemm:xstarindts}
With probability at least $1-\delta$, it holds that $\x^* \in \Dts$ for any $t>0$ that satisfies:
\begin{align}
\frac{\epsilon^2}{4\beta_{t}} \geq \sigma^2_{g,t-1}(\x^*), \label{eq:thirdlower}
\end{align}
\end{lemma}

From Lemma \ref{lemm:xstarindts}, it suffices to establish an appropriate upper bound on the RHS of \eqref{eq:thirdlower} to determine the duration of the first phase, i.e., $T'$.

A positive semi-definite kernel function $k_g : \mathbb{R}^d \times \mathbb{R}^d \rightarrow \mathbb{R}$ is associated with a feature map $\varphi_g: \mathbb{R}^d \rightarrow \mathcal{H}_{k_g}$ that maps the vectors in the primary space to a reproducing kernel Hilbert space (RKHS). In terms of the mapping $\varphi_g$, the kernel function $k_g$ is defined by:
\begin{align}
    k_g(\x,\x') = \varphi_g(\x)^T \varphi_g(\x'), ~\forall \x,\x' \in \mathbb{R}^d.
\end{align}

Let $d_g$ denote the dimension of $\mathcal{H}_{k_g}$ (potentially infinite) and define the $t\times d_g$ matrices $\mathbf{\Phi}_{g,t} := [\varphi_g(\x_1),\ldots,\varphi_g(\x_t)]^T$ at each round $t$.  Using this notation, we can rewrite $\sigma_{g,t}^2$ as (see Appendix \ref{sec:concrete} for details):
\begin{align}
    \sigma_{g,t}^2(\x) = \sigma^2 \varphi_g(\x)^T(\mathbf{\Phi}_{g,t}^T\mathbf{\Phi}_{g,t}+\sigma^2\I)^{-1}\varphi_g(\x). \label{eq:goodexpression}
\end{align}

 In the following two subsections, we discuss how this expression helps us control $\sigma_{g,t-1}^2(\x^*)$ for $t\geq T'+1$.
 Depending on the type of kernel functions $k_g$ and their corresponding $d_g$, we derive different expressions for $T'$ in Theorems \ref{thm:finite} and \ref{thm:infinite}.

\subsection{Constraint with finite-dimensional RKHS}\label{sec:finite}
In this section we consider $g$ with finite dimensional RKHS. Linear and polynomial kernels are special cases of these types of functions. For a linear kernel $k_g(\x,\y) = \x^T\y$ and a polynomial kernel $k_g(\x,\y) = (\x^T\y+1)^{p}$, the corresponding $d_g$ is  $d$ and $\binom {d+p}{d}$, respectively \cite{pham2013fast}.

Let $\bar \x \sim\text{Unif}(\Dw)$ be a $d$-dimensional random vector uniformly distributed in $\Dw$. At rounds $t\in[T']$, \SGPUCB~chooses safe iid actions $\x_t \simiid \bar \x$. 
We denote the covariance matrix of $\varphi_g(\bar \x)$ by $\mathbf{\Sigma}_{g}=\E[\varphi_g(\bar \x)\varphi_g(\bar \x)^T] \in \mathbb{R}^{d_g \times d_g}$. A key quantity in our analysis is the minimum eigenvalue of $\mathbf{\Sigma}_{g}$ denoted by:
\begin{align}
    \la_- := \lamin ({\mathbf{\Sigma}_{g}}).
\end{align}

Regarding the definition of $\sigma_{g,t}^2(\x)$ in \eqref{eq:goodexpression}, we show that if $\la_->0$,  $\sigma_{g,t-1}^2(\x^*)$ can be controlled for all $t\geq T'+1$ by appropriately lower bounding the minimum eigenvalue of the Gram matrix $\mathbf{\Phi}_{g,T'}^T\mathbf{\Phi}_{g,T'}$, which is 
 possible due  to the randomness of chosen actions in the first phase.

\begin{lemma}\label{lemm:finite}
 Assume $d_g<\infty$, $\la_- >0$, and $\x \in \Dc_0$. 
Then, for any $\delta \in (0,1)$, provided $T' \geq t_{\delta} :=  \frac{8}{\la_-}\log(\frac{ d_g}{\delta})$, the following holds with probability at least $1-\delta$,
\begin{align}
	  \lamin\left(\mathbf{\Phi}_{g,T'}^T\mathbf{\Phi}_{g,T'}+\sigma^2\I\right)\geq \sigma^2+\frac{\la_-T'}{2}.
	\end{align}
Consequently,
	$ \sigma_{g,t-1}^2(\x^*) \leq \frac{2\sigma^2}{2\sigma^2+ \la_-T'},
	$
for all $t\geq T'+1$.
\end{lemma}

We present the proof in Appendix \ref{sec:prooffinite}.

Combining Lemmas \ref{lemm:xstarindts} and \ref{lemm:finite} gives the desired value of $T'$ that guarantees $\x^*\in \Dts$ for all $t\geq T'+1$ with high probability. Putting these together, we conclude the following regret bound for constraint with corresponding finite-dimensional RKHS.

\begin{theorem}[Regret bound for $g$ with finite dimensional RKHS]\label{thm:finite}
Let the same assumptions as in Lemma \ref{lemm:finite} hold. Let $t_{\epsilon} := \frac{8\sigma^2\beta_T}{ \la_-\epsilon^2}$
and define $T' := t_{\epsilon} \vee t_{\delta}$. Then for sufficiently large $T$ and any $\delta \in (0,1/3)$, with probability at least $1-3\delta$:
\begin{align}
    R_T \leq BT'+ \sqrt{C_1 T \beta_T \gamma_T},
\end{align}
where $C_1 = 8/\log(1+\sigma^{-2})$. 

\end{theorem}

See Appendix \ref{sec:prooffinite1} for proof details.

\noindent{\bf Linear Kernels.}
We highlight the setting where $f$ and $g$ are associated with linear kernels as a special case.  In this setting the primal space $\mathbb{R}^d$ and the corresponding RKHS are the same. Let $k$ be a linear kernel with mapping $\varphi_g: \mathbb R^d \rightarrow \mathcal{H}_k= \mathbb R ^d$, $\X _t= \mathbf{\Phi}_{t} =[\x_1,\ldots,\x_t]^T$, and $\y$ be the corresponding observation vector. Therefore, we have $\mu_t(\x) = \x^T{\hat  \theta}_t$ where $\hat \theta_t=(\X_t^T\X_t+\sigma^2\I)^{-1}\X_t^T\y$.
  We drive the following from \eqref{eq:goodexpression}:
\begin{align}\label{eq:variancelinear}
 \sigma_{t}^2(\x) =\sigma^2\|\x\|_{\Ab_t^{-1}},
\end{align}
where $\Ab_t = \X_{t}^T\X_{t}+\sigma^2\I$. Thus, we observe the close relation in these notations with that in Linear stochastic bandits settings, (e.g. see \cite{Dani08stochasticlinear,abbasi2011improved}). As such, our setting is an extension to \cite{amani2019linear}, where linear loss and constraint functions have been studied (albeit in a frequentist setting).

\subsection{Constraint with infinite-dimensional RKHS}\label{sec:infinite}
Now we provide regret guarantees for a more general case where the underlying RKHS corresponding to $g$ can be infinite-dimensional. 

In the infinite-dimensional RKHS setting, controlling $\sigma_{g,t-1}(\x^*)$ for $t \geq T'+1$ can be challenging. To address this issue,  we focus on stationary kernels, i.e., $k_g(\x,\y) = k_g(\x-\y)$\footnote{This property holds for a wide variety of kernels including Exponential, Gaussian, Rational quadratic, etc.},  and  apply a \emph {finite basis approximation} in our analysis. Particularly, we consider $\tilde \varphi_g: \mathbb{R}^d \rightarrow \mathbb{R}^{D_g}$ which maps the input to a lower-dimensional Euclidean inner product space with dimension $D_g$ such that:
\begin{align}
    k_g(\x,\y) \approx \tilde \varphi_g(\x)^T\tilde \varphi_g(\y).
\end{align}

\begin{definition}[\textbf {$(\epsilon_0,D_g)$-uniform approximation}]
Let $k_g: \mathbb{R}^d \times \mathbb{R}^d \rightarrow \mathbb{R}$ be a stationary kernel, then the inner product $\tilde \varphi_g(\x)^T\tilde \varphi_g(\y)$ in $\mathbb{R}^{D_g}$, $(\epsilon_0,D_g)$-uniformly approximates $k_g(\x,\y)$ if and only if:
\begin{equation}
    \sup_{\x , \y \in \Dc_0} |\tilde \varphi_g (\x)^T\tilde \varphi_g (\y)-k(\x,\y)|\leq \epsilon_0.\nn
\end{equation}
\end{definition}



Due to the infinite dimensionality of $\mathcal{H}_{k_g}$, there is no notion for minimum eigenvalue of $\mathbf{\Phi}_{g,T'}^T\mathbf{\Phi}_{g,T'}$. Hence, we adopt an $(\epsilon_0,D_g)$-unifrom approximation to bound $\sigma_{g,t-1}^2(\x^*)$ for all $t\geq T'+1$ by lower bounding the minimum eigenvalue of the approximated $D_g \times D_g$ matrix $\mathbf{\tilde \Phi}_{g,T'}^T\mathbf{\tilde\Phi}_{g,T'}$ instead. The argument follows the same procedure as in Lemma $\ref{lemm:finite}$, other than an error bound on $\sigma_{g,t-1}^2(\x^*)$ caused by the $(\epsilon_0,D_g)$-unifromly approximation is required.

We consider $\tilde \varphi_g(.)$ to be an ($\epsilon_0,D_g$)-uniform approximation and denote the covariance matrix of $\tilde \varphi_g(\bar \x)$ by $ \mathbf{\tilde \Sigma}_g=\E[\tilde \varphi_g(\bar \x) \tilde \varphi_g(\bar \x)^T] \in \mathbb{R}^{D_g \times D_g}$ with minimum eigenvalue:
\begin{align}
   \tilde \la_- := \lamin({\mathbf{\tilde \Sigma}_{g}}).\label{eq:latil}
\end{align}

\begin{lemma}\label{lemm:infinite} 
 Assume that $d_g = \infty$, $k_g$ is a stationary kernel, and $\tilde \la_-$ defined in \eqref{eq:latil} is positive. Fix $\delta, \epsilon_0 \in (0,1)$. Then, it holds with probability at least $1-\delta$ for all $t\geq T'+1$,
	\begin{align}
	    \sigma_{g,t-1}^2(\x^*) \leq \frac{2\sigma^2}{2\sigma^2+\tilde \la_-T'}+\frac{4{t}^3\epsilon_0}{\sigma^2}, 
	\end{align}
provided that $T' \geq \tilde t_{\delta} :=  \frac{8}{ \tilde \la_-}\log(\frac{ D_g}{\delta})$.
\end{lemma}

Technical details on how $\tilde \varphi_g$ analytically helps us obtain this upper bound on $\sigma_{g,t-1}^2(\x^*)$ for all $t\geq T'+1$ by lower bounding the minimum eigenvalue of $\mathbf{\tilde \Phi}_{g,T'}^T\mathbf{\tilde\Phi}_{g,T'}$ are deferred to Appendix \ref{sec:proofinfinite}. Putting these together, we obtain the regret bound for constraint with corresponding infinite-dimensional RKHS in the following theorem.

\begin{theorem}[Regret bound for $g$ with infinite dimensional RKHS]\label{thm:infinite}
Assume there exists an $(\epsilon_0,D_g)$-uniform approximation of stationary kernel $k_g$ with $0 < \epsilon_0 \leq \frac{\epsilon^2 \sigma^2}{32{T}^3\beta_T}$ for which $\tilde \la_-$ defined in \eqref{eq:latil} is positive. Let $\tilde t_{\epsilon} := \frac{16\sigma^2\beta_T}{\tilde \la_-\epsilon^2}$ and $\tilde t_{\delta} :=  \frac{8}{ \tilde \la_-}\log(\frac{ D_g}{\delta})$ and define $T' := \tilde t_{\epsilon} \vee \tilde t_{\delta}$. Then, for sufficiently large $T$ and any $\delta \in (0,1/3)$, with probability at least $1-3\delta$:
\begin{align}
    R_T \leq BT'+ \sqrt{C_1 T \beta_T \gamma_T},
\end{align}
where $C_1 = 8\log(1+\sigma^{-2})$.


\end{theorem}

Complete proof is given in Appendix \ref{sec:prooffinite1}.


Depending on the feature map approximation $\tilde \varphi_g$, the dimension $D_g$ can be appropriately chosen as a function of the algorithm's inputs $\epsilon$, $\delta$ and $d$ to control the accuracy of the approximation. We emphasize that our analysis is not restricted to specific approximations (see Appendix \ref{sec:RFF} for details). We focus on the \emph{Quadrature Fourier features} (QFF) studied by \cite{mutny2018efficient} who show that for any stationary kernel $k$ on $\mathbb{R}^d$ whose inverse Fourier transform decomposes product-wise, i.e., $p(\omega)=\prod_{i=1}^d p_j(\omega_j)$, we can use Gauss-Hermite quadrature \cite{hildebrand1987introduction} to approximate it. The results in \cite{mutny2018efficient} imply that the QFF uniform approximation error $\epsilon_0$ decreases exponentially with $D_g$. More concretely, in this case, $D_g = \mathcal{O}\left((d+\log(d/\epsilon_0))^d\right)$ features are required to obtain an $\epsilon_0$-accurate approximation of the SE kernel $k_g$.





\begin{figure*}
  \centering
     \begin{subfigure}[b]{0.32\textwidth}
         \centering
         \includegraphics[width=\textwidth]{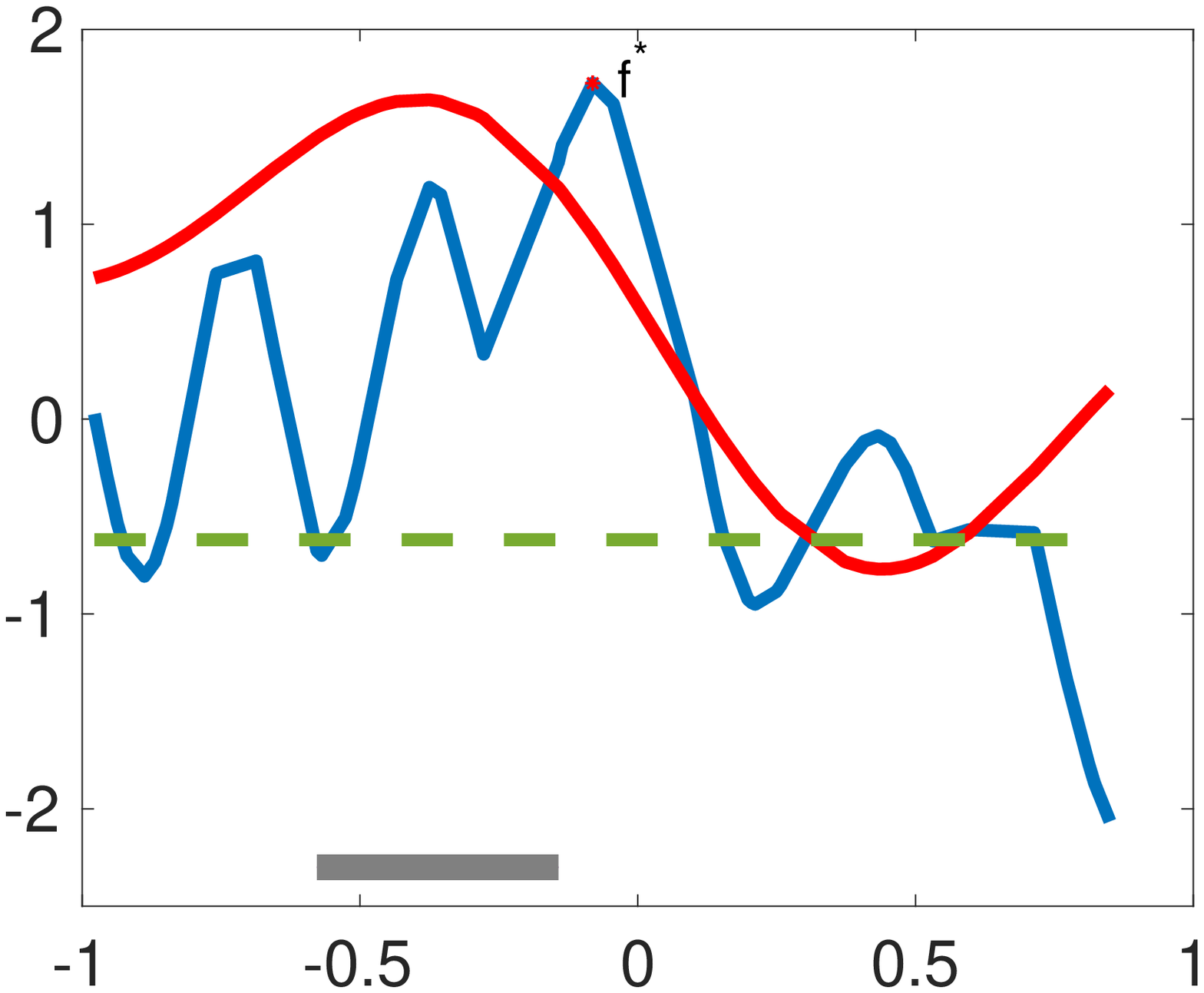}
        \caption{True $f$,$g$ and given safe seed set.}
         \label{subfig:setting}
     \end{subfigure}
    \hfill 
     \begin{subfigure}[b]{0.32\textwidth}
         \centering
         \includegraphics[width=\textwidth]{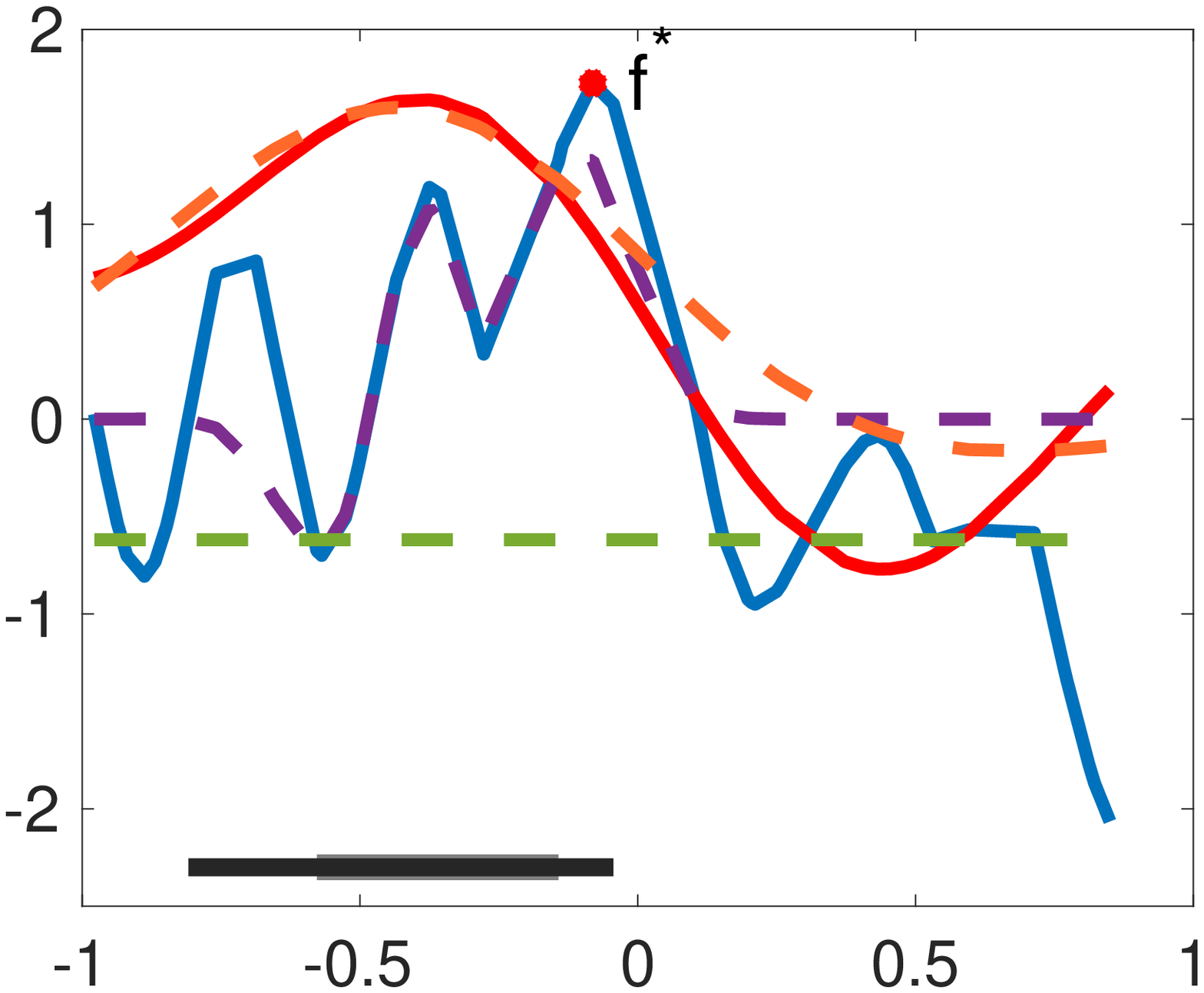}
         \caption{$t=50$}
          \label{subfig:time50}
     \end{subfigure}
     \hfill 
     \begin{subfigure}[b]{0.32\textwidth}
         \centering
         \includegraphics[width=\textwidth]{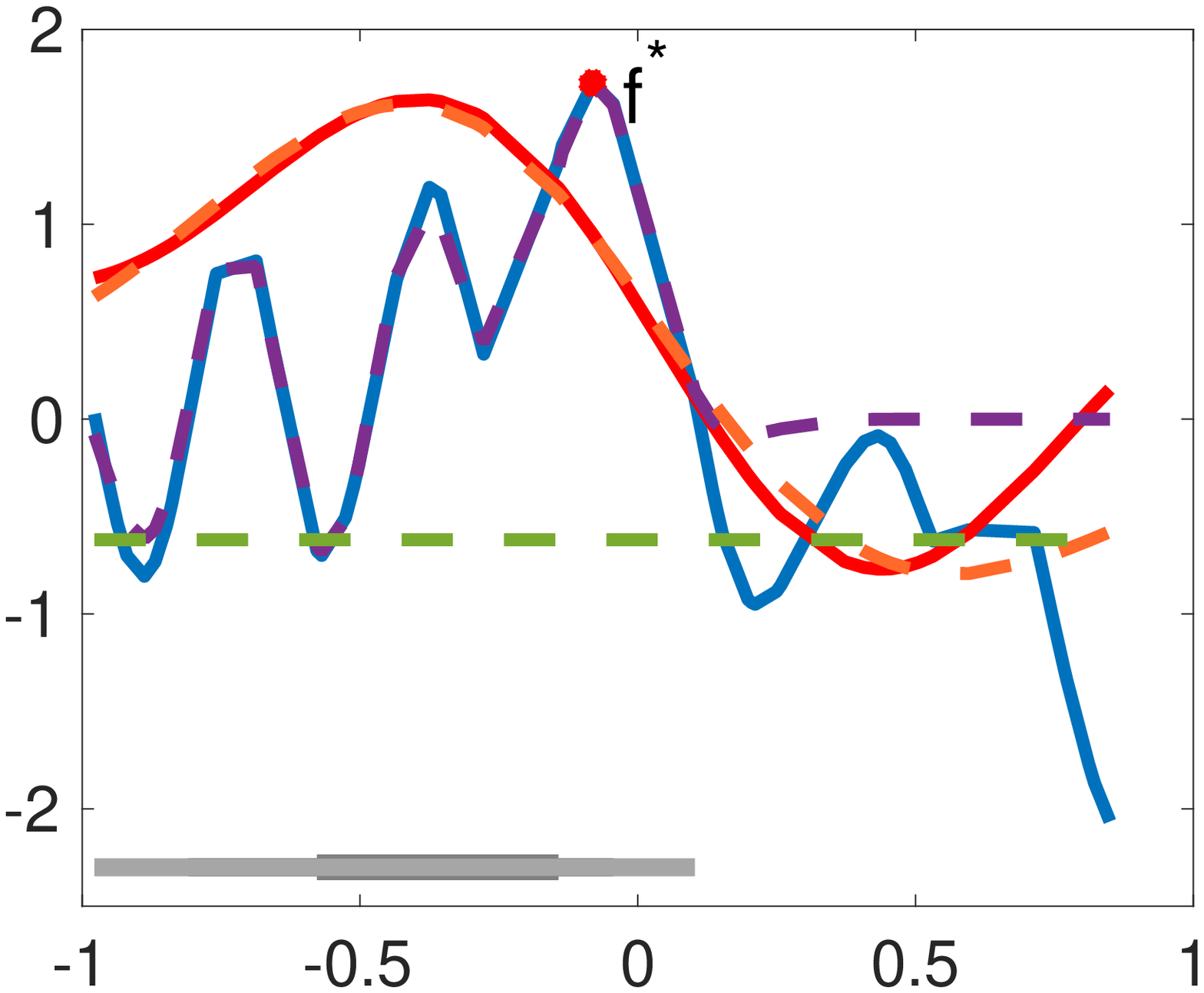}
         \caption{$t=200$}
          \label{subfig:time20}
     \end{subfigure}
  \caption{Illustration of \SGPUCB: (a) The blue and red solid lines denote the unknown reward function $f$ and constraint function $g$, respectively. The dashed green line represents the threshold $h+\epsilon$, the gray bar shows the safe seed set $\Dc^w$, and the red star is the optimum value of $f$ through $\Des$. (b,c) The dashed purple and orange lines are the estimated GP mean functions corresponding to $f$ and $g$, respectively at rounds 50 and 200.}
  \label{fig:illustration}
\end{figure*}

\section{Comparison to existing algorithms}\label{sec:counterexample}

A few remarks regarding the differences between our algorithm and existing work on safe-GP optimization are in order.
We first remark on the assumptions placed on the safe seed set $\Dc^w$ in our work, which might appear restrictive when compared to  those in the closely related works of \cite{Krause,sui2018stagewise}.
Specifically, our theoretical guarantees require the safe seed set $\Dc^w$  to satisfy assumptions put forth in Theorems \ref{thm:finite} and \ref{thm:infinite}, which would ensure that $\la_->0$ and $\tilde \la_->0$. For  instances with dimension $d_g<\infty$ discussed in Section \ref{sec:finite}, a sufficient condition that guarantees $\la_->0$  is that $\Dc^w$ contains at least $d_g$ actions, such that their maps $\varphi_g(.)$ into each corresponding RKHS form linearly independent vectors. For example, for linear constraints the size of the seed set needs only to be 	linear in the dimension $d$. Hence, our analysis suggests that Safe GP learning is easy when the safety constraint is \emph{simple} (e.g. linear/polynomial kernels). However, for instances with dimension $d_g=\infty$ discussed in Section \ref{sec:infinite}, the assumption $\tilde \la_->0$ holds if at least $D_g$ actions with linear independent corresponding $\tilde \varphi_g(.)$ exist in $\Dc^w$. As a corollary, employing QFF for SE kernels in the analysis requires $\Dc^w$ to contain at least $\Oc\left(\left(d+\log(T/\epsilon)\right)^d\right)$ actions. 

In comparison, the safe-GP algorithms proposed by \cite{Krause,sui2018stagewise,berkenkamp2016bayesian} can start from a safe seed set of arbitrary size; However, this comes at a costs: 1) There is \emph{no} guarantee that they are able to explore the entire space to reach a sufficiently expanded safe set that includes $\x^*$; 2) They require the functions $f$ and $g$ to be Lipschitz continuous with known constants. More specifically, in all the above mentioned work a one-step reachablity operator for a single constraint function $g$ is defined as follows: $R_{\epsilon}(S_0):= S_0 \cup \{\x \in \Dc_0| \exists \x'\in S_0, g(\x')-\epsilon-Ld(\x,\x')\geq h\}$, where $L$ is the Lipschitz constant corresponding to the constraint function $g$. 
Then they define an $\epsilon$-reachable safe set by $\bar R_{\epsilon}(S_0):=\lim_{n\rightarrow \infty} R_{\epsilon}^n(S_0)$, where $R_{\epsilon}^n(S_0) := \underbrace{R_{\epsilon}(R_{\epsilon}\ldots (R_{\epsilon}}_{n~ \rm times}(S_0))\ldots)$ is an $n$-step reachability operator starting from the safe seed set $S_0$  \cite{Krause}. The above stated $\epsilon$-reachable safe set clearly depends on $S_0$. The optimization benchmark in \cite{Krause,berkenkamp2016bayesian} is $f_{\epsilon}^*(S_0)=\max_{\x \in \bar R_{\epsilon}(S_0)}f(\x)$, which varies on a case by case basis depending on $S_0$. 

Instead, our benchmark is $f_{\epsilon}^*=\max_{\x \in \Des}f(\x)$ which satisfies $f_{\epsilon}^*\geq f_{\epsilon}^*(S_0)$ since $\bar R_{\epsilon}(S_0) \subseteq \Des$ for any choice of $S_0$. Similarly, the optimization goal of StageOpt \cite{sui2018stagewise} is approaching $\argmax_{\x \in R_{\epsilon}^{t^*}(S_0)}$ for an arbitrary safe seed set $S_0$, where $t^*$ is the round at which the first phase of StageOpt, ends under the condition $\max{w_t(\x)_{\x\in G_t}\leq \epsilon}$, where $G_t$ is the set of potential expander points that is created at each round $t$ and $w_t(\x)$ is the width of confidence interval of constraint function $g(\x)$ at round $t$ (see \cite{sui2018stagewise}). Hence, as pointed out in \cite{sui2018stagewise}, given an arbitrary seed set, it is not guaranteed that
they will be able to discover the globally optimal decision
$\x^*$, e.g. if the safe region around $\x^*$
is topologically separate from that of $S_0$. On the other hand, given its stronger assumptions on the safe seed set, our algorithm does not suffer from this issue.

Our second remark is concerning the fundamentally different goals of our algorithm versus that of \cite{Krause,sui2018stagewise,berkenkamp2016bayesian}. Unlike \SGPUCB, the  proposed algorithms in the latter works are not focused on regret minimization; rather,  their focus is on best arm identification through safe exploration, i.e., providing convergence guarantees to the  reachable optimal solution defined in the previous remark. We refer the reader to Appendix \ref{sec:morediscuss} for more clarifications on the algorithmic design differences of \SGPUCB~and the algorithm studied by \cite{Krause} and their implications on cumulative regret analysis.


\begin{figure*}
  \centering
  \begin{subfigure}[b]{0.32\textwidth}
         \centering
         \includegraphics[width=\textwidth]{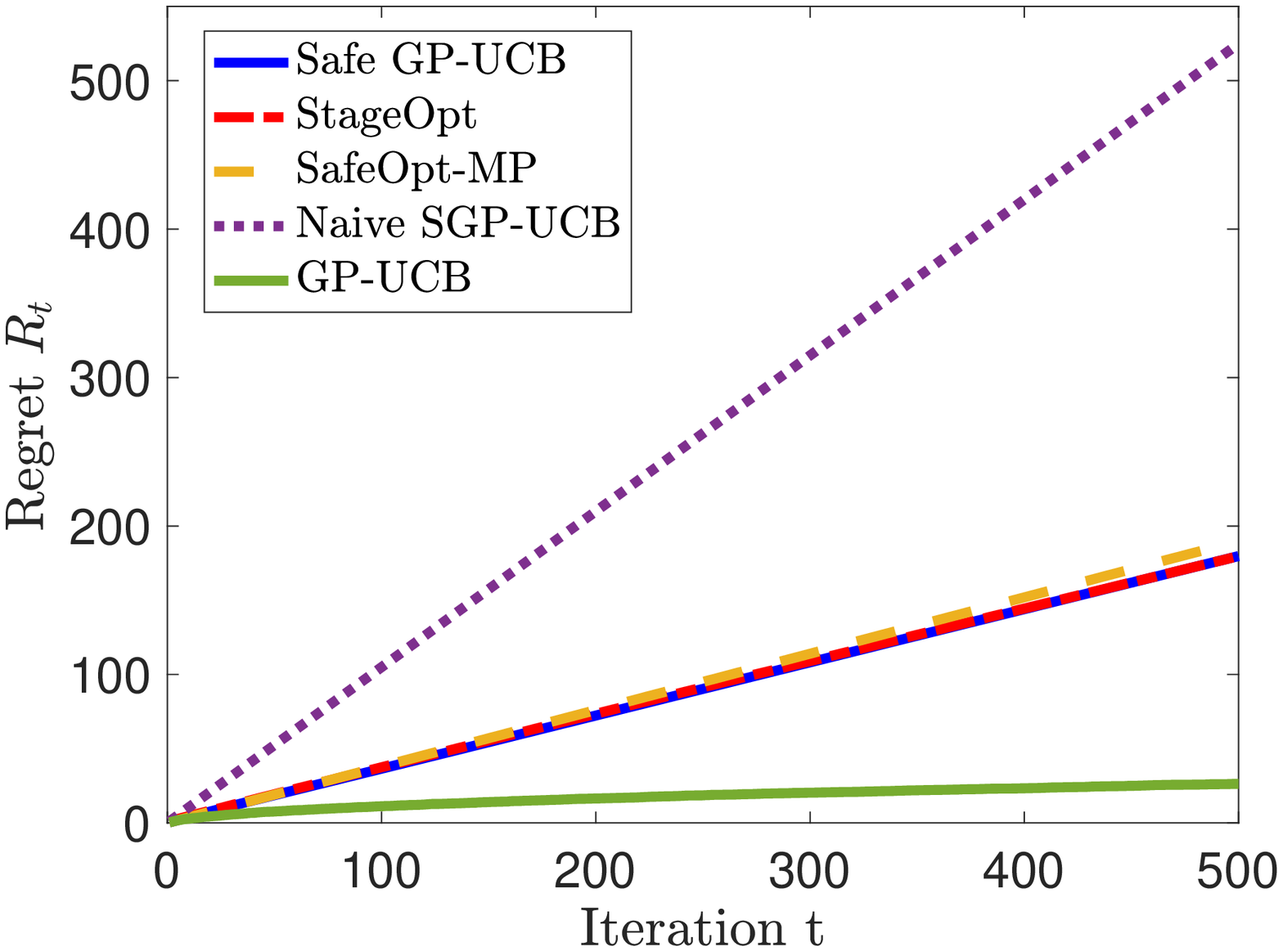}
        \caption{$1\leq|\Dc^w|\leq 10$.}
          \label{subfig:badbad}
     \end{subfigure}
    \hfill 
     \begin{subfigure}[b]{0.32\textwidth}
         \centering
         \includegraphics[width=\textwidth]{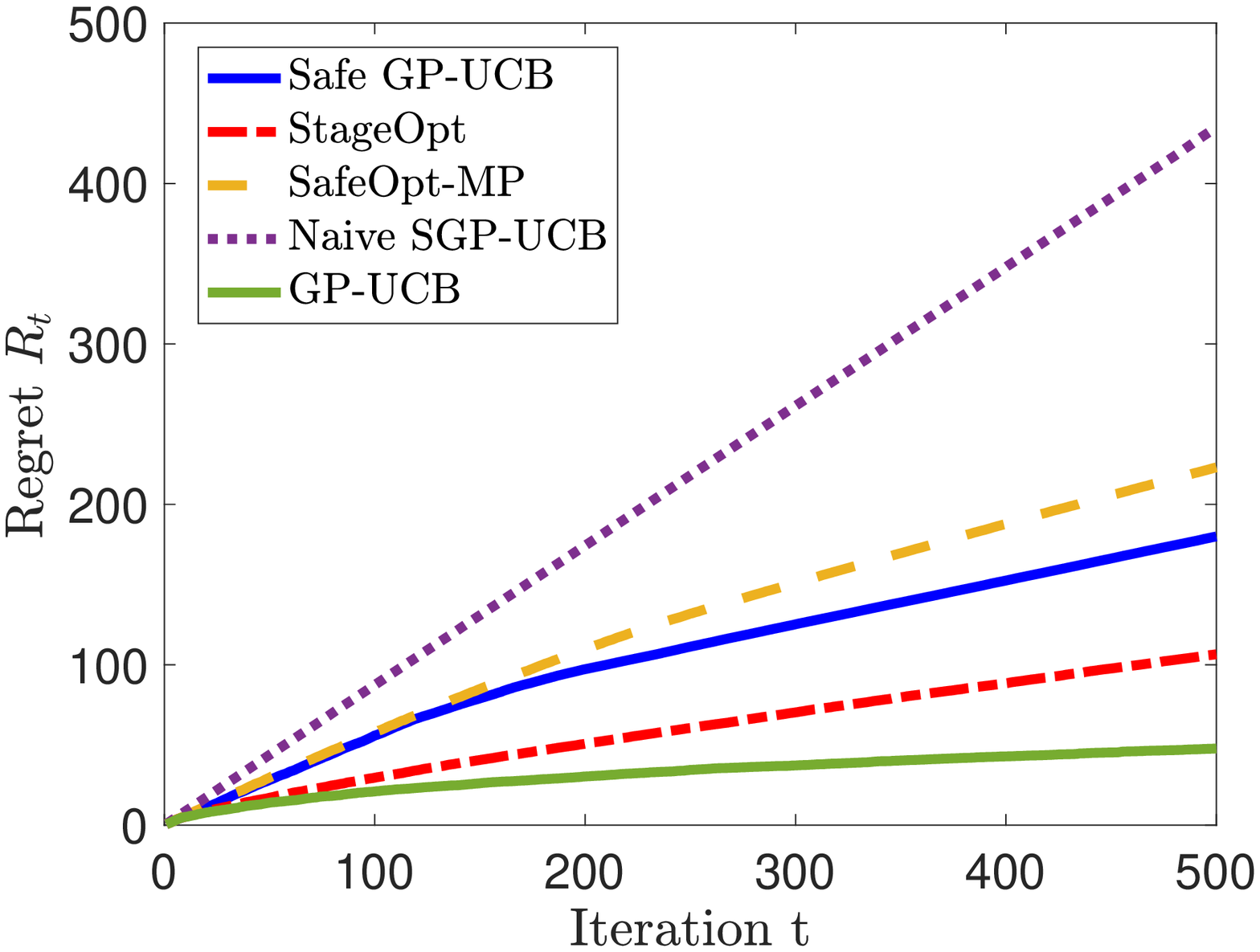}
       \caption{$11\leq|\Dc^w|\leq 20$.}
          \label{subfig:badgood}
     \end{subfigure}
     \hfill
     \begin{subfigure}[b]{0.32\textwidth}
         \centering
         \includegraphics[width=\textwidth]{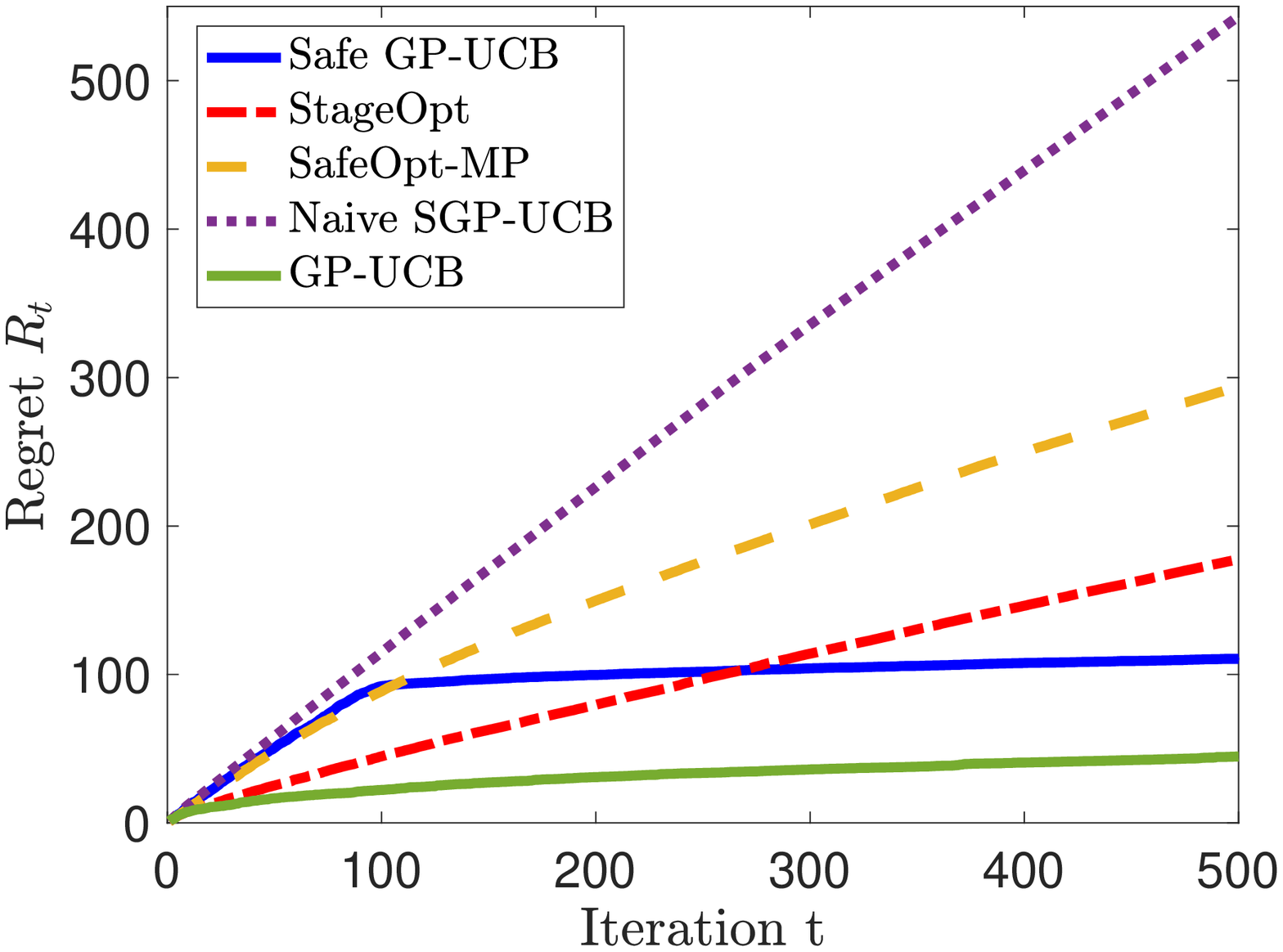}
         \caption{$21\leq|\Dc^w|\leq 25$.}
         \label{subfig:goodbad}
     \end{subfigure}

\caption{Regret comparison}
  \label{fig:regretcomparison}
\end{figure*}

\section{Experiments}\label{sec:simulations}
In this section, we analyze \SGPUCB~through numerical evaluations on synthetic data. We first give an illustration of how our algorithm performs by depicting the estimated $f$ and $g$ and the expanded safe sets at certain rounds. The second experiment seeks to compare \SGPUCB's performance against a number of other existing algorithms. 

In Figure \ref{fig:illustration}, we give an illustration of \SGPUCB's performance. For the sake of visualization, we implement the algorithm in a 1-dimensional space and connect the data points since we find it instructive to also depict estimates of $f$ and $g$ as well as the growth of the safe sets. The algorithm starts the first phase by sampling actions at random from a given safe seed set. After 50 rounds, in Figure \ref{subfig:time50}, the safe set has sufficiently expanded such that the optimal action $\x^*$ lies within the $\Dst_{50}$. Figure \ref{subfig:time20} shows the expansion of the safe set after 200 rounds, which still includes $\x^*$.

Figure \ref{fig:regretcomparison} compares the average per-step regret of \SGPUCB~against a number of closely related algorithms over 30 realizations (we include the error bars in figures presented in Appendix \ref{sec:appendixsimulation}). In particular, we compare against 1) StageOpt \cite{sui2018stagewise}; 2) SafeOpt-MC \cite{berkenkamp2016bayesian} that generalizes SafeOpt \cite{Krause} to settings with multiple constraints possibly different than the objective function $f$;
3) A heuristic variant of GP-UCB, which proceeds the same as \SGPUCB~except that there is no exploration phase, i.e., $T'=0$; 4) The standard GP-UCB with oracle access to the safe set.



We evaluated regret on synthetic settings with reward and constraint functions corresponding to SE kernels with hyper-parameters $1$ and $0.1$, respectively. Parameters $T=500$, $\delta = 0.01$, $\epsilon=0.01$ and $\sigma=0.1$ have been chosen in all settings. The decision sets are time-independent sets of 100 actions sampled uniformly from the unit ball in $\mathbb{R}^2$.
We implemented all algorithms by starting from the same seed set, i.e., $\Dc^w=S_0$.
Figure \ref{fig:regretcomparison} highlights the key role of the seed set's size that is discussed in detail in Section \ref{sec:counterexample}. Figure \ref{subfig:goodbad} shows that once the safe seed set contains enough actions, \SGPUCB~outperforms SafeOpt-MP and StageOpt whose $\epsilon$-reachable set (i.e., $\bar R_{\epsilon}(\Dc^w)$ in SafeOpt-MP and $R_{\epsilon}^{t^*}(\Dc^w)$ in StageOpt) do not include the true globally optimal $\x^*$ considered in this paper. We also implemented \SGPUCB~for settings where $\Dc^w$ has relatively small number of safe actions. The results given in Figure \ref{subfig:badbad} show the poor performance of \SGPUCB~which is expected since $\Dc^w$ is not large enough to explore the whole space for the purpose of safe set expansion. In Figure \ref{subfig:badgood}, the regret curves are plotted for instances where $\Dc^w$ is not large enough to reach the sufficiently expanded safe set including $\x^*$, but it is also not too small to get the expansion process stuck. In these instances, StageOpt performs well compared to \SGPUCB~on average. What is common in all figures is the poor performance of Naive \SGPUCB~(almost linear regret) compared to the others since it is never able to expand the safe set properly.
When implementing SafeOpt-MC, we took the results of Theorem 1 in \cite{berkenkamp2016bayesian} into account. We found $t^*$ numerically and modified the sampling rule after $t^*$ as follows: $\x_t:= \argmax_{\x \in S_t}\ell_{f,t}(\x)$.

Another issue worth highlighting regrading implementation of SafeOpt-MC and StageOpt is construction of the safe sets $S_t$. In our experiments, we relied on the exact definition of $S_t$ suggested in \cite{berkenkamp2016bayesian, sui2018stagewise}, which depends on the Lipschitz constant of $g$. While we numerically calculate the Lipschitz constant to have a fair comparison, \cite{berkenkamp2016bayesian,sui2018stagewise} use only the GP model to ensure safety in their numerical experiments. As such, they construct $S_t$ in the same way as we form $\Dc_{t}^{\rm s}$ in \eqref{eq:safeset}. However, since the provided guarantees in these works are obtained with respect to the optimal action through an $\epsilon$-reachable set, i.e., $\bar R_{\epsilon}(S_0)$, which clearly depends on $S_0$, this modification disregards the role of $S_0$ in the provided theoretical results.

As is the case in our proposed algorithm, StageOpt also proceeds in two distinct phases and the duration of the first phase is an input to the algorithm which needs to be specified. An interesting observation is that there are similarities between the first phase duration suggested for StageOpt and that introduced in our paper. These similarities mostly come from their dependence on  parameters such as $\beta_T$ and $\epsilon$.
In our experiments, we did not rely on the value of $T'$ that the theoretical results suggest. For both implementations, we stopped the first phase when the safe region plateaued for at least 20 iterations, and also hard capped $T'$ at 100 iterations (a similar approach was adopted by \cite{sui2018stagewise}).

\section{Discussion and  future work}\label{sec:conclusions}

We studied a safe stochastic bandit optimization problem where the unknown payoff and constraint functions are sampled from GPs. We  proposed \SGPUCB~which is comprised of two phases: (i) a pure-exploration phase that speeds up learning of the safe set; (ii) a safe exploration-exploitation phase that focuses on  regret minimization. We balanced the two-fold challenge of minimizing regret and expanding the safe set by properly choosing the duration of the first phase  $T'$. Our analysis suggests that the type of kernels associated with the constraint functions plays a critical role in tuning $T'$ and consequently affects the regret bounds.  
We used QFF \cite{mutny2018efficient} as a tool to facilitate our analysis in settings with constraint function with infinite-dimensional RKHS. Beyond analysis, it is interesting to employ such approximations or other approaches like variational inference introduced by \cite{huggins2019scalable} to further overcome \emph{computational} associated with solving \eqref{eq:decisionrule}.

Several issues remain to be studied.
While our algorithm is the first providing regret guarantees for safe GP optimization, it is not clear whether it is the best to apply. The answer could depend on the application. Hence, numerical comparisons on real application-specific data is worth investigating. More importantly, the other issue that needs to be addressed is that the existing guarantees (either in terms of cumulative regret, simple regret or optimization gap) for all safe-GP optimization algorithms, suffer from loose constants that make such comparisons hard. Indeed evaluating the performances of all these four algorithms in numerical experiments requires us to resort to empirical tuning of parameters like $T'$, which is an important challenge to overcome. 

\section{Acknowledgement}
This research is supported by UCOP grant LFR-18-548175 and NSF grant 1847096.




\bibliographystyle{apalike}
\bibliography{Paper_main}

\newpage
\appendix


\section{Proof of Lemma \ref{lemm:xstarindts}} \label{sec:proofxstarindts}
In this section we prove the Lemma \ref{lemm:xstarindts}, which states the condition under which it is guaranteed that with probability at least $1-\delta$ it holds that $\x^* \in \Dts$.
\begin{proof}
In order to check whether $\x^*\in \Dts$ holds, we can equivalently see if:
\begin{align}\label{eq:checkpoint}
    \mu_{g,t-1}(\x^*)-\beta_{t}^{1/2}\sigma_{g,t-1}(\x^*) \geq h.
\end{align}
If we lower-bound the LHS of \eqref{eq:checkpoint} using the definition of confidence interval $Q_{g,t}(\x^*)$, we obtain:
\begin{align}
    &g(\x^*)-2\beta_{t}^{1/2}\sigma_{g,t-1}(\x^*) \geq h, \nn \\
    ~~\Leftrightarrow~~&g(\x^*)-h\geq 2\beta_{t}^{1/2}\sigma_{g,t-1}(\x^*). \label{eq:secondlower}
\end{align}
Since $\x^*\in\Dc_\epsilon^{\text{s}}$, lower bounding the LHS of \eqref{eq:secondlower} gives:
\begin{align}
    \epsilon \geq 2\beta_{t}^{1/2}\sigma_{g,t-1}(\x^*).\label{eq:thirdlower1}
\end{align}
Since each confidence interval $Q_{g,t}(\x)$ is built to contain the $g(\x)$ with high probability, it is clear that \eqref{eq:checkpoint} is satisfied whenever \eqref{eq:thirdlower1} is true.
\end{proof}

\section{Proof of Lemma \ref{lemm:finite}} \label{sec:prooffinite}
In order to bound the minimum eigenvalue of the Gram matrices $\mathbf{\Phi}_{g,T'}^T\mathbf{\Phi}_{g,T'}$, we use the Matrix Chernoff Inequality \cite{tropp2015introduction}.
\begin{theorem}[Matrix Chernoff Inequality,~\cite{tropp2015introduction}]\label{thm:chernoff}
Consider a finite sequence $\{\X_k\}$ of independent, random, symmetric matrices in $\mathbb{R}^d$. Assume that $\lamin(\X_k) \geq 0$ and $\lamax(\X_k) \leq L$ for each index $k$. Introduce the random matrix $\Y = \sum _k \X_k$. Let $\mu_{min}$ denote the minimum eigenvalue of the expectation $\E[\Y]$,
\begin{equation}
    \mu_{\rm min} = \lamin\left(\E[\Y]\right)=\lamin\left(\sum _k E[\X_k]\right). \nonumber
\end{equation}
Then, for any $\varepsilon\in(0,1)$, it holds,
\begin{equation}
    \Pr\left( \lamin(\Y)\leq\varepsilon\mu_{\rm min}\right)\leq d\cdot \exp\left(-(1-\varepsilon)^2\frac{\mu_{\rm min}}{2L}\right). \nonumber
\end{equation}
\end{theorem}

\begin{proof}[Proof of Lemma \ref{lemm:finite}.]~
Let $\X_{t} = \varphi_g(\x_t) \varphi_g(\x_t)^T$ for $t \in [T']$, such that each $\X_{t}$ is a symmetric matrix with $\lamin(\X_t)\geq 0$ and $\lamax(\X_t)\leq 1$. In this notation, $
    \mathbf{\Phi}_{g,T'}^T\mathbf{\Phi}_{g,T'}+\sigma^2\I = \sum_{t = 1} ^{T'} \X_{t}+\sigma^2 \I.$ 
We compute:
\begin{align}
    \mu_{\rm min}  &:= \lamin\left(\sum _{t=1}^{T'} \E[\X_{t}]\right)\nonumber
    =\lamin\left(\sum _{t=1}^{T'} \E[\varphi_g(\x_t) \varphi_g(\x_t)^T]\right)\nonumber
    =\lamin\left(T' \mathbf{\Sigma}_g \right)\nonumber
    =\lambda_- T'. \nonumber
\end{align}
Thus, Theorem \ref{thm:chernoff} implies the following for any $\varepsilon\in[0,1)$:
\begin{align} \label{eq:chernoff_in_lemma}
    \Pr\left[ \lamin(\sum_{t = 1} ^{T'} \X_{t})\leq\varepsilon\la_-T'\right]\leq d_g\cdot \exp\left(-(1-\varepsilon)^2\frac{\la_-T'}{2}\right). 
\end{align}
To complete the proof of the lemma, simply choose $\varepsilon = 0.5$ and $T' \geq \frac{8L^2}{\la_-}\log(\frac{d_g}{\delta})$ in  \eqref{eq:chernoff_in_lemma}. This gives
\begin{align}
        \Pr\left[ \lamin\left(\mathbf{\Phi}_{g,T'}^T\mathbf{\Phi}_{g,T'}+\sigma^2\I\right)\geq \sigma^2+\frac{\la_-T'}{2},\right]\geq 1-\delta.
\end{align}
Using this high probability lower bound on $\lamin\left(\mathbf{\Phi}_{g,T'}^T\mathbf{\Phi}_{g,T'}+\sigma^2\I\right)$ and the fact that $\lamin\left(\mathbf{\Phi}_{g,t-1}^T\mathbf{\Phi}_{g,t-1}+\sigma^2\I\right)\geq\lamin\left(\mathbf{\Phi}_{g,T'}^T\mathbf{\Phi}_{g,T'}+\sigma^2\I\right)$ for all $t\geq T'+1$, we can easily obtain the desired bound on $\sigma_{g,t-1}^2(\x^*)$ for all $t\geq T'+1$ as follows:
\begin{align}
    \sigma_{g,t-1}^2(\x^*) &= \sigma^2 \varphi_g(\x^*)^T(\mathbf{\Phi}_{g,t-1}^T\mathbf{\Phi}_{g,t-1}+\sigma^2\I)^{-1}\varphi_g(\x^*)\nn \\
    &\leq \sigma^2 \| \varphi_g(\x^*)\|^2 \lamax\left((\mathbf{ \Phi}_{g,t-1}^T\mathbf{ \Phi}_{g,t-1}+\sigma^2\I)^{-1}\right)\nn \\
    &\leq \frac{\sigma^2}{\lamin\left(\mathbf{ \Phi}_{g,t-1}^T\mathbf{ \Phi}_{g,t-1}+\sigma^2\I\right)} \nn \\
     &\leq \frac{\sigma^2}{\lamin\left(\mathbf{ \Phi}_{g,T'}^T\mathbf{ \Phi}_{g,T'}+\sigma^2\I\right)} \nn \\
    &\leq \frac{2\sigma^2}{2\sigma^2+\la_-T'}. \label{eq:sigmastar}
\end{align}
\end{proof}

\section{Proof of Lemma \ref{lemm:infinite}} \label{sec:proofinfinite}
In this section, we present the proof of Lemma \ref{lemm:infinite}.

First, we bound the ${\sigma^{2}_{g,t}}(\x)-{\tilde \sigma^{2}_{g,t}}(\x)$(\cite{mutny2018efficient}), where ${\tilde \sigma^{2}_{g,t}}(\x)$ is the approximated posterior variance.

\begin{lemma}[Approximation of posterior variance, \cite{mutny2018efficient}]
Let the $\tilde \varphi_g(.)^T \tilde \varphi_g(.)$ $(\epsilon_0,D_g)$-uniformly approximates the kernel $k_g$.
Then,
\begin{align}
    {\sigma^{2}_{g,t-1}}(\x^*)\leq {\tilde \sigma^{2}_{g,t-1}}(\x^*)+\frac{4{t}^3\epsilon_0}{\sigma^2}, \quad \forall t\geq T'+1. \label{eq:sigmaerror}
\end{align}
\end{lemma}

\noindent{\bf Completing the proof of Lemma \ref{lemm:infinite}.}
\begin{proof}

In order to complete the proof of Lemma \ref{lemm:infinite}, we employ a similar technique as in the proof of Lemma \ref{lemm:finite}. In this direction, we bound $\la_{min}\left(\mathbf{\tilde \Phi}_{g,T'}^T\mathbf{\tilde \Phi}_{g,T'}+\sigma^2\I\right)$ using Theorem \ref{thm:chernoff} such that with probability at least $1-\delta$:
\begin{align}
    \la_{min}\left(\mathbf{\tilde \Phi}_{g,T'}^T\mathbf{\tilde \Phi}_{g,T'}+\sigma^2\I\right) \geq \sigma^2+\frac{\tilde\la_-T'}{2}, 
\end{align}
provided that  $T' \geq  \frac{8}{\tilde\la_-}\log(\frac{ D_g}{\delta})$. Therefore, we can conclude for all $t\geq T'+1$:
\begin{align}
    \tilde \sigma_{g,t-1}^2(\x^*) &= \sigma^2 \tilde \varphi_g(\x^*)^T(\mathbf{\tilde \Phi}_{g,t-1}^T\mathbf{\tilde \Phi}_{g,t-1}+\sigma^2\I)^{-1}\tilde \varphi_g(\x^*) \nn \\
    &\leq \sigma^2 \|\tilde \varphi_g(\x^*)\|^2 \lamax\left((\mathbf{\tilde \Phi}_{g,t-1}^T\mathbf{\tilde \Phi}_{g,t-1}+\sigma^2\I)^{-1}\right)\nn \\
    &= \frac{\sigma^2}{\lamin\left(\mathbf{\tilde \Phi}_{g,t-1}^T\mathbf{\tilde \Phi}_{g,t-1}+\sigma^2\I\right)} \nn \\
    &\leq \frac{\sigma^2}{\lamin\left(\mathbf{\tilde \Phi}_{g,T'}^T\mathbf{\tilde \Phi}_{g,T'}+\sigma^2\I\right)} \nn \\
    &\leq \frac{2\sigma^2}{2\sigma^2+\tilde\la_-T'}. \label{eq:tildesigmastar}
\end{align}
Note that for a QFF and RFF map $ \|\tilde \varphi_g(\x^*)\| = 1$.

Now we combine \eqref{eq:sigmaerror} and \eqref{eq:tildesigmastar} to conclude that for $T' \geq  \frac{8}{\tilde\la_-}\log(\frac{ D_g}{\delta})$ and all $t\geq T'+1$ with probability at least $1-\delta$:
\begin{align}
    {\sigma^{2}_{g,t-1}}(\x^*)\leq \frac{2\sigma^2}{2\sigma^2+\tilde\la_-T'}+\frac{4{T}^3\epsilon_0}{\sigma^2},
\end{align}
as desired.

\end{proof}

\section{Proof of Theorems \ref{thm:finite} and \ref{thm:infinite}}\label{sec:prooffinite1}

First, we decompose the cumulative regret $R_T$ as follows:
\begin{align}
R_T = \underbrace{\sum_{t=1}^{T'}r_t}_{\rm Term~I}\,+\,\underbrace{\sum_{t=T'+1}^{T}r_t}_{\rm Term~II}.
\end{align}
The bound on the second term is standard in the literature (see for example \cite{srinivasgaussian}), but we provide the necessary details for completeness.

\begin{lemma}\label{lemm:thirdlemma}
Let $\beta_t$ be defined as in Theorem \ref{thm:confidence_interval} and $\x^*\in \Dts$ for $t\geq T'+1$ . Then, for any $\delta \in (0,1)$ with probability at least $1-\delta$:
\begin{align}
\sum_{t=T'+1}^T r_t^2 \leq C_1\beta_T\gamma_T,
\end{align}
\end{lemma}
where $C_1 = 8/\log(1+\sigma^{-2})$.
\begin{proof}
The proof is mostly adapted from \cite{srinivasgaussian}. We start with analysis of $r_t$. If $\x^*\in \Dts$ for $t\geq T'+1$, the following holds for all $t \geq T'+1$ with probability at least $1-\delta$:
\begin{align}
    r_t &= f(\x^*)-f(\x_t) \leq u_{f,t}(\x^*)-f(\x_t)\leq u_{f,t}(\x_t)-f(\x_t) \leq 2\beta_t^{1/2}\sigma_{f,t-1}(\x_t), \label{eq:rt2}
\end{align}
where the second inequality follows from the definition of decision rule in \eqref{eq:decisionrule} and the fact that $\x^*\in \Dts$ for $t\geq T'+1$. For the last inequality we used the definition of the confidence interval $Q_{f,t}$ in \eqref{eq:confidence_interval1}.

It follows from \eqref{eq:rt2} that with probability at least $1-\delta$:
\begin{align}
    \sum_{t=T'+1}^T r_t^2 &\leq 4\beta_T\sigma^2\sum_{t=T'+1}^T\sigma^{-2}\sigma^2_{t-1}(\x_t) \nn \\
    &\leq 4\beta_T\sigma^2C_2\sum_{t=T'+1}^T\log(1+\sigma^{-2}\sigma^2_{t-1}(\x_t)) \label{eq:sumrt2}\\
    &\leq 4\beta_T\sigma^2C_2\sum_{t=1}^T\log(1+\sigma^{-2}\sigma^2_{t-1}(\x_t))\label{eq:thirdineq}\\
    &= 8\beta_T\sigma^2C_2I(\y_T;\f_T)\leq C_1\beta_T\gamma_T,
\end{align}
where in \eqref{eq:sumrt2} we used $s^2 \leq C_2\log(1+\sigma^{-2})$ for any $s \in [0,\sigma^{-2}]$ when $C_2 = \sigma^{-2}/\log(1+\sigma^{-2}) \geq 1$ and \eqref{eq:thirdineq} follows from the definition of the mutual information between function values $\f_T$ and the observations $\y_T$.
\end{proof}

\begin{lemma}\label{lemm:B}
With probability at least $1-\delta$:
\begin{align}
\sum_{t=1}^{T'}r_t \leq BT',
\end{align}
where $B:= C\sqrt{2\ell d} ~\textup{diam} (\Dc_0)/ \delta$ for some positive universal constant $C>0$ if the function $f$ is associated with an RBF kernel with parameter $\ell$, otherwise $B :=  2\sqrt{2\log(2|\Dc_0|)}/\delta$.
\end{lemma}
\begin{proof}
By the GP assumption, $f(\x)$ is a random gaussian $|\Dc_0|\times 1$ vector with mean 0 and covariance matrix $\K=[{k_f}_{\x,\x' \in \Dc_0}]$. First, we assume $k_f$ is an RBF kernel. For an RBF kernel $k_f$ with parameter $\ell$, \cite{vershynin2018high} implies that for some universal constant $C_1>0$:

\begin{align}
\|f(\x)-f(\x')\|_{\psi_2} \leq C_1 \sqrt{k_f(\x,\x)-k_f(\x,\x')+k_f(\x',\x')}&\leq C_1\sqrt{2(1-e^{\ell (\|\x-\x'\|)})}\|\x-\x'\|_2 \nn \\
&\leq C_1\sqrt{2\ell}\|\x-\x'\|_2.
\label{eq:gaussianwidth}
\end{align}
Let $w(\Dc_0)$ be the \emph{gaussian width} \cite{vershynin2018high} of $\Dc_0$. Then for some universal constant $C>0$:
\begin{align}
\Pr(\max_{\x\in \Dc_0}|f(\x)|\geq M)\leq \frac{\mathbb{E}(\max_{\x\in \Dc_0}|f(\x)|)}{M} \leq C\sqrt{2\ell}w(\Dc_0)/M \leq C\sqrt{2\ell d} ~\text{diam} (\Dc_0)/2M. \label{eq:rbfkernel}
\end{align}
The first inequality follows from the Markov inequality. The second inequality is an application of \cite{vershynin2018high} combined with \eqref{eq:gaussianwidth}. Finally, the last inequality holds due to the fact that $w(\Dc_0)\leq \frac{\text{diam}(\Dc_0)}{2}\sqrt{d}$. Therefore, \eqref{eq:rbfkernel} implies that:

\begin{align}
\Pr\left(\max_{\x\in \Dc_0}|f(\x)|\geq C\sqrt{2\ell d} ~\text{diam} (\Dc_0)/2\delta\right)\leq \delta,
\end{align}

For a general setting, where $k_f$ is not an RBF kernel, we have:
\begin{align}
\Pr(\max_{\x\in \Dc_0}|f(\x)|\geq M)\leq \frac{\mathbb{E}(\max_{\x\in \Dc_0}|f(\x)|)}{M} \leq\max_{\x,\x'\in \Dc_0}k_f(\x,\x')\sqrt{2\log(2|\Dc_0|)}/M\leq \sqrt{2\log(2|\Dc_0|)}/M
\end{align}
For the second inequality see \cite{rigollet201518} and the last inequality holds because $k_f(\x,\x')\leq 1$ for all $\x,\x' \in \Dc_0$.

Hence, we deduce that:
\begin{align}
\Pr(\max_{\x\in \Dc_0}|f(\x)|\geq \sqrt{2\log(2|\Dc_0|)}/\delta)\leq \delta,
\end{align}
We finalize the proof as follows:
\begin{align}
\Pr(\max_{\x,\x' \in \Dc_0}|f(\x)-f(\x')|\leq B) \geq \Pr(2\max_{\x\in \Dc_0}|f(\x)|\leq B)\geq 1-\delta,
\end{align}
provided that $B \geq C\sqrt{2\ell d} ~\text{diam} (\Dc_0)/\delta$ if $k_f$ is RBF kernel with parameter $\ell$ and $B \geq 2\sqrt{2\log(2|\Dc_0|)}/\delta$.
\end{proof}

Next, we show that when the Theorems \ref{thm:finite} and \ref{thm:infinite} choices of $T'$ are combined with Lemmas \ref{lemm:finite} and \ref{lemm:infinite}, respectively, \eqref{eq:thirdlower} holds with probability at least $1-\delta$ for all $t\geq T'+1$. 

\noindent {\bf Completing the proof of Theorem \ref{thm:finite}.}
Let $T$ be sufficiently large such that $T \geq T'$. For all $t\geq T'+1$, Lemma \ref{lemm:finite} implies:
\begin{align}
    4\beta_{t}\sigma_{g,t-1}^2(\x^*) \leq \frac{8\sigma^2\beta_{T}}{2\sigma^2+\la_-T'} \leq \epsilon ^2 , \label{eq:thm2}
\end{align}
where the last inequality follows from the the definition of $T'$ and the fact that $T' \geq \frac{8\sigma^2\beta_T}{\la_-\epsilon^2}$. Hence, as stated in Lemma \ref{lemm:xstarindts}, \eqref{eq:thm2} is equivalent to $\x^* \in \Dts$ for all $t\geq T'+1$ with probability at least $1-\delta$.

We are now ready to complete the proof of Theorem \ref{thm:finite}. Let $T$ sufficiently large such that:
\begin{align}
T>T':= t_{\epsilon} \vee t_{\delta}.\label{eq:T0_app_thm2}
\end{align}
We combine Lemmas \ref{lemm:finite}, \ref{lemm:thirdlemma} and \ref{lemm:B} to conclude that with probability at least $1-3\delta$:
\begin{align*}
 R_T = \sum_{t=1}^{T'} r_t+ \sum_{t=T'+1}^T r_t 
&\leq BT'+\sqrt{C_1T\beta_T\gamma_T}.
\end{align*}

\noindent {\bf Completing the proof of Theorem \ref{thm:infinite}.}
Let $T$ be sufficiently large such that $T \geq T'$. Employing Lemma $\ref{lemm:infinite}$, for $t\geq T+1$ we have:
\begin{align}
 4\beta_{t}\sigma_{g,t-1}^2(\x^*)  \leq 4\beta_{T}\left(\frac{2\sigma^2}{2\sigma^2+\tilde\la_-T'}+\frac{4{T}^3\epsilon_0}{\sigma^2}\right), \label{eq:proofthm4}
\end{align}
where, in the first inequality we used $\beta_{T'+1}\leq \beta_T$. Note that assumptions $\epsilon_0 \leq \frac{\epsilon^2 \sigma^2}{32{T}^3\beta_T}$ and $T'\geq \frac{16\sigma^2\beta_T}{\tilde\la_-\epsilon^2}$ combined with \eqref{eq:proofthm4} guarantees  $\eqref{eq:thirdlower}$. Hence with probability at least $1-\delta$, $\x^* \in \Dts$ for all $t\geq T'+1$. This fact allows us to bound Term II in \eqref{eq:Terms} using Lemma \ref{lemm:thirdlemma}. 

Hence, we combine Lemmas \ref{lemm:finite}, \ref{lemm:thirdlemma} and \ref{lemm:B} to conclude that with probability at least $1-3\delta$:

\begin{align*}
 R_T = \sum_{t=1}^{T'} r_t+ \sum_{t=T'+1}^T r_t 
&\leq BT'+\sqrt{C_1T\beta_T\gamma_T}.
\end{align*}

\section{Posterior variance}\label{sec:concrete}
In this section we show why Eqn. \eqref{eq:goodexpression} holds.
Since the matrices $\mathbf{\Phi}_{g,t}^T\mathbf{\Phi}_{g,t}+\sigma^2\I$ and $\mathbf{\Phi}_{g,t}\mathbf{\Phi}_{g,t}^T+\sigma^2\I$ are positive definite and $(\mathbf{\Phi}_{g,t}^T\mathbf{\Phi}_{g,t}+\sigma^2\I)\mathbf{\Phi}_{g,t}^T =\mathbf{\Phi}_{g,t}^T(\mathbf{\Phi}_{g,t}\mathbf{\Phi}_{g,t}^T+\sigma^2\I)$, we get:
\begin{align}
    \mathbf{\Phi}_{g,t}^T(\mathbf{\Phi}_{g,t}\mathbf{\Phi}_{g,t}^T+\sigma^2\I)^{-1}=(\mathbf{\Phi}_{g,t}^T\mathbf{\Phi}_{g,t}+\sigma^2\I)^{-1}\mathbf{\Phi}_{g,t}^T. \label{eq:firstone}
\end{align}
Also from the definition of  $k_{g,t}(\x) := \mathbf{\Phi}_{g,t} \varphi_g(\x)$ for all $\x \in \Dc_0$, we deduce:
\begin{align}
(\mathbf{\Phi}_{g,t}^T\mathbf{\Phi}_{g,t}+\sigma^2\I) \varphi_g(\x)= \mathbf{\Phi}_{g,t}^Tk_{g,t}(\x)+\sigma^2 \varphi_g(\x). \label{eq:secondone}
\end{align}
Combining \eqref{eq:firstone} and \eqref{eq:secondone}, we get:
\begin{align}
    \varphi_g(\x) = \mathbf{\Phi}_{g,t}^T(\mathbf{\Phi}_{g,t}\mathbf{\Phi}_{g,t}^T+\sigma^2\I)^{-1}k_{g,t}(\x)+\sigma^2(\mathbf{\Phi}_{g,t}^T\mathbf{\Phi}_{g,t}+\sigma^2\I)^{-1}\varphi_g(\x),
\end{align}
which gives:
\begin{align}
    \varphi_g(\x)^T \varphi_g(\x) = k_{g,t}(\x)^T(\mathbf{\Phi}_{g,t}\mathbf{\Phi}_{g,t}^T+\sigma^2\I)^{-1}k_{g,t}(\x)+\sigma^2\varphi_g(\x)^T(\mathbf{\Phi}_{g,t}^T\mathbf{\Phi}_{g,t}+\sigma^2\I)^{-1}\varphi_g(\x). \label{eq:finalstep}
\end{align}

At the final step, \eqref{eq:finalstep} implies:
\begin{align}
    \sigma^2\varphi_g(\x)^T(\mathbf{\Phi}_{g,t}^T\mathbf{\Phi}_{g,t}+\sigma^2\I)^{-1}\varphi_g(\x) = k_g(\x,\x)-k_{g,t}(\x)^T(\K_{g,t}+\sigma^2\I)^{-1}k_{g,t}(\x)=\sigma^2_{g,t}(\x).
\end{align}

\section{Uniform approximations}\label{sec:RFF}
\begin{definition}[QFF approximation]
For SE kernel the Fourier transform is $p(\omega) = \left(\frac{l}{\sqrt{2\pi}} \right)^d e^-\frac{l^2\|\omega\|_2^2}{2}$. If $\mathcal{D}_0 = [0,1]^d$, the SE kernel $k_g$ is approximated as follows. Choose $ \bar D_g \in \mathbb{N}$ and $ D_g = \bar D_g ^d$, and construct the $2D_g$-dimensional feature map:
\begin{equation}
    \tilde \varphi_g(\x)_i=
    \begin{cases}
      \sqrt{\nu(\omega_i)} \cos\left(\frac{\sqrt{2}}{l}w_i^T\x\right) &\text{if}~1\leq i\leq D_g, \\
      \sqrt{\nu(\omega_{i-D_g})} \sin\left(\frac{\sqrt{2}}{l}w_{i-D_g}^T\x\right) &\text{if}~D_g+1\leq i\leq 2D_g,
    \end{cases}
\end{equation}
where $\{\omega_1,\ldots\omega_{D_g}\} = \overbrace{A_{\bar D_g}\times \ldots\times A_{\bar D_g}}^{d~\rm times}$, $A_{\bar D_g}$ is the set of $\bar D_g$ (real) roots of the $\bar D_g$-th Hermite polynomial $H_{\bar D_g}$, and $\nu(\z)=\prod_{j=1}^d \frac{2^{\bar D_g-1} \bar D_g!}{\bar D_g^2 H_{\bar D_g-1}}(z_j)^2$ for all $\z \in
\mathbb{R}^d$.
\end{definition}

\begin{lemma}\cite{mutny2018efficient}\label{lemm:qff}
Let $k_g$ be the SE kernel and $\Dc_0 =[0,1]^d$. Then, for an $(\epsilon_0,D_g)$-uniform QFF approximation in $\mathbb{R}^{D_g}$ it holds that:
\begin{align} 
\epsilon_0 \leq d 2^{d-1} \frac{1}{\sqrt{2}\bar D_g^{\bar D_g}}\left(\frac{e}{4l^2}\right)^{\bar D_g}.
\end{align}
\end{lemma}
Lemma \ref{lemm:qff} implies that if $\bar D_g > 1/l^2$, the uniform approximation error $\epsilon_0$ decreases exponentially with $D_g$.

Our analysis and results are not limited to only stationary kernels whose Fourier transform decomposes product-wise, i.e., $p(\omega)=\prod_{i=1}^d p_j(\omega_j)$. There exists other uniform approximations adapted for a broader range of kernels.
Specifically, for any stationary kernel, we can adapt the so-called {\it Random Fourier Features} (RFF) mapping introduced by \cite{rahimi2008random}.

\begin{definition}[RFF approximation]
For any stationary kernel $k_g$ if $\mathcal{D}_0 = [0,1]^d$, the $D_g$-dimensional feature map is constructed by:
\begin{align}
    \tilde \varphi_g(\x) =& \sqrt{2/D_g}[\sin(\omega_1^T\x),\cos(\omega_1^T\x),\ldots,\sin(\omega_{d_g/2}^T\x),\cos(\omega_{d_g/2}^T\x)]^T,
\end{align}
where $\omega_j, j \in [d_g/2]$ are i.i.d random variables in $\mathbb{R}^d$. 
\end{definition}
In the following theorem, we see how $D_g$ is chosen for an $(\epsilon_0,D_g)$-uniform RFF approximation.

\begin{theorem}\cite{rahimi2008random}\label{thm:rahimi}
Let $k_g$ be an stationary kernel $\mathcal{D}_0 = [0,1]^d$. Then, for Random Fourier Features mapping $\tilde \varphi_g$ and any $\delta \in (0,1)$ it holds with probability at least $1-\delta$:
\begin{align}
    \sup_{\x , \y \in \Dc_0} |\tilde \varphi_g (\x)^T\tilde \varphi_g (\y)-k_g(\x,\y)|\leq \epsilon_0,\nn
\end{align}
provided that ,
\begin{align} 
D_g:=D_g(\delta,\epsilon_0) \geq \frac{8(d+2)}{\epsilon_0^2}\log\! \left(\!\frac{16 \rho_{g}\sqrt{m}}{\epsilon_0\sqrt{\delta}}\!\right), \label{eq:Di}
\end{align}
where $\rho_{g}^2$ is the trace of the Hessian of $k_g$ at 0.
\end{theorem}

While the uniform approximation error of QFF
decreases exponentially with the size of the linear basis, applying the standard RFF \cite{rahimi2008random} for any stationary kernels implies $D_g = \mathcal{O}\left(\frac{d}{\epsilon_0^2}\right)$ number of features are required for an $(\epsilon_0,D_g)$-uniform RFF approximation; in other words, the uniform approximation error of RFF decreases with the inverse square root of the basis dimension. See Appendix \ref{sec:RFF} for details on QFF and RFF.

In comparison, QFF scales unfavorably with the dimensionality of the model. Hence, on one hand, QFF is unsuitable for an arbitrary high dimensional kernel approximation. The strengths of QFF manifest on problems with a low dimension or a low effective dimension. On the other hand, adapting RFF results in drastically large $D_g$ when a very small $\epsilon_0$ is required to control the accuracy of the approximation.
\begin{figure}
    \centering
         \includegraphics[width=0.38\textwidth]{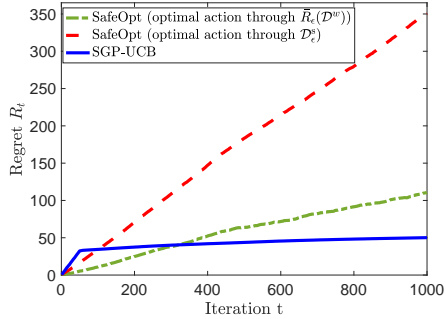}
        \caption{Average regret of \SGPUCB~and SafeOpt with linear kernel.}\label{fig:safeoptlinear}
         \end{figure}

\section{Comparison to existing algorithms}\label{sec:morediscuss}

In this section we address the design of safe-GP optimization algorithms for the purpose of best arm identification (studied in \cite{Krause,sui2018stagewise}) versus that of regret minimization (studied in this paper) and highlight why regret guarantees are not the focus of the former.

A popular sampling criteria for best arm identification, which is also adopted for the safe-GP optimization setting by  \cite{Krause}, relies on a purely exploratory approach referred to as the \emph{uncertainty sampling} (a.k.a. maximum variance) rule. Here, the decision maker would select actions from a safe subset of true safe set, $G_t \cup M_t$, with the highest variance of the GP estimate:
\begin{equation}\label{eq:uncertaintysampling}
    \x_t = \argmax_{\x \in G_t \cup     M_t}\sigma_{t-1}(\x),
\end{equation}
where $G_t \subseteq S_t$ and $M_t \subseteq S_t$ are the set of potential expanders and maximizers, respectively, and $S_t$ is the expanded safe set at round $t$ that is constructed based on the knowledge of Lipschitz constant of $f$ (see \cite{Krause} for more details).

A general observation about the uncertainty sampling rule adopted in \eqref{eq:uncertaintysampling} is that while it is a provably good way to explore a function, it is not well suited to control regret, i.e., identifying points $\x$ where $f(\x)$ is large in order to concentrate sampling there without unnecessarily exploration. A relevant paper that has also highlighted this aspect is  \cite{contal2013parallel}, which has applied the maximum variance rule \eqref{eq:uncertaintysampling} as a pure exploration approach for GP optimization (in the unconstrained setting), and yet proves theoretical upper bounds on the regret with batches of size $K$.
To do so, they introduce
the Gaussian Process Upper Confidence Bound and Pure Exploration
algorithm (GP-UCB-PE) which combines the UCB strategy and Pure Exploration. GP-UCB-PE combines the benefits of the
UCB policy with pure exploration queries which is based on uncertainty sampling rule in the same batch of $K$ evaluations
of $f$. While only relying on the maximum variance as the sampling criterion results in a greedy pure exploratory algorithm, the UCB strategy has been used in parallel with the pure explorative rule to obtain regret bound.

For the purpose of further clarification, we would also like to numerically highlight the unsuitability of the uncertainty sampling rule \eqref{eq:uncertaintysampling} for  regret minimization. Consider the following example. In a relaxed version of Safe GP, let the objective function $f$ be associated with linear kernels, the true safe set be known to the algorithm and contain all the standard basis vectors, $e_i$, whose only non-zero element is the $i$-th element which is 1. At each round $t$, uncertainty sampling \eqref{eq:uncertaintysampling}
maximizes the estimated variance, which is $\|\x\|_{A_{t-1}^{-1}}$ according to \eqref{eq:variancelinear}, over the given safe set. It can be shown that at each round $t$, it equivalently selects an eigenvector of $A_{t-1}$ corresponding to its minimum eigenvalue. Hence, the standard basis vectors, $e_i$, are repeatedly selected, resulting in a linear regret. In order to illustrate this issue, we implemented SafeOpt on 20 instances where the kernels are linear and the true safe sets $\Des$ are as explained above (and of course unknown to the algorithm). We also evaluated SafeOpt when the cumulative regret is obtained with respect to the benchmark considered in \cite{Krause}, i.e., $f_{\epsilon}^*=\max_{\x \in \bar R_{\epsilon}(\Dc^w)}f(\x)$. Figure \ref{fig:safeoptlinear} compares the average regret curves of \SGPUCB~and SafeOpt, with respect to both $f_{\epsilon}^*=\max_{\x \in \bar R_{\epsilon}(\Dc^w)}f(\x)$ and true benchmark and $f_{\epsilon}^*=\max_{\x \in \Des} f(\x)$. Please note that in this experiment, we estimated $\bar R_{\epsilon}(\Dc^w)$ by $R_{\epsilon}^{1000}(\Dc^w)$ . Figure \ref{fig:safeoptlinear} highlights the poor performance of SafeOpt compared to that of \SGPUCB~when the regret is obtained with respect to true $\x^*$. Let us however reiterate that SafeOpt was {\it not designed for regret minimization} to begin with.

A final remark is concerning a potential modification to \SGPUCB~that can improve performance of the first phase but it is unclear how one can analyze this theoretically. Due to pure-explorative behaviour of uncertainty sampling, it might be an appropriate alternative for sampling actions form the safe seed set $\Dc^w$ in the first phase to explore the function $g$. In Figure \ref{fig:uncertainuniform}, we depict the average regret curves of \SGPUCB~over 20 instances where $f$ and $g$ are sampled from GPs with SE kernels with hyper parameters 1 and 0.1, respectively. The curves highlight the performance of \SGPUCB~when two different exploration approaches, uniform and uncertainty sampling, are applied in the first phase. 

\begin{figure}
    \centering
         \includegraphics[width=0.38\textwidth]{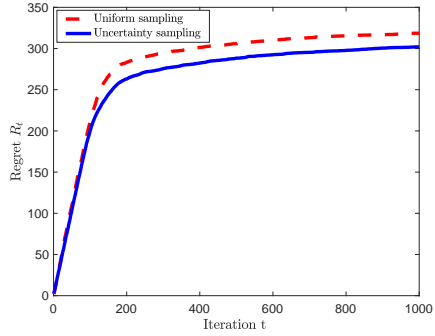}
        \caption{Average regret of \SGPUCB: uniform and uncertainty sampling in the first phase.}\label{fig:uncertainuniform}
         \end{figure}

\begin{figure}
  \centering
  \begin{subfigure}[b]{0.32\textwidth}
         \centering
         \includegraphics[width=\textwidth]{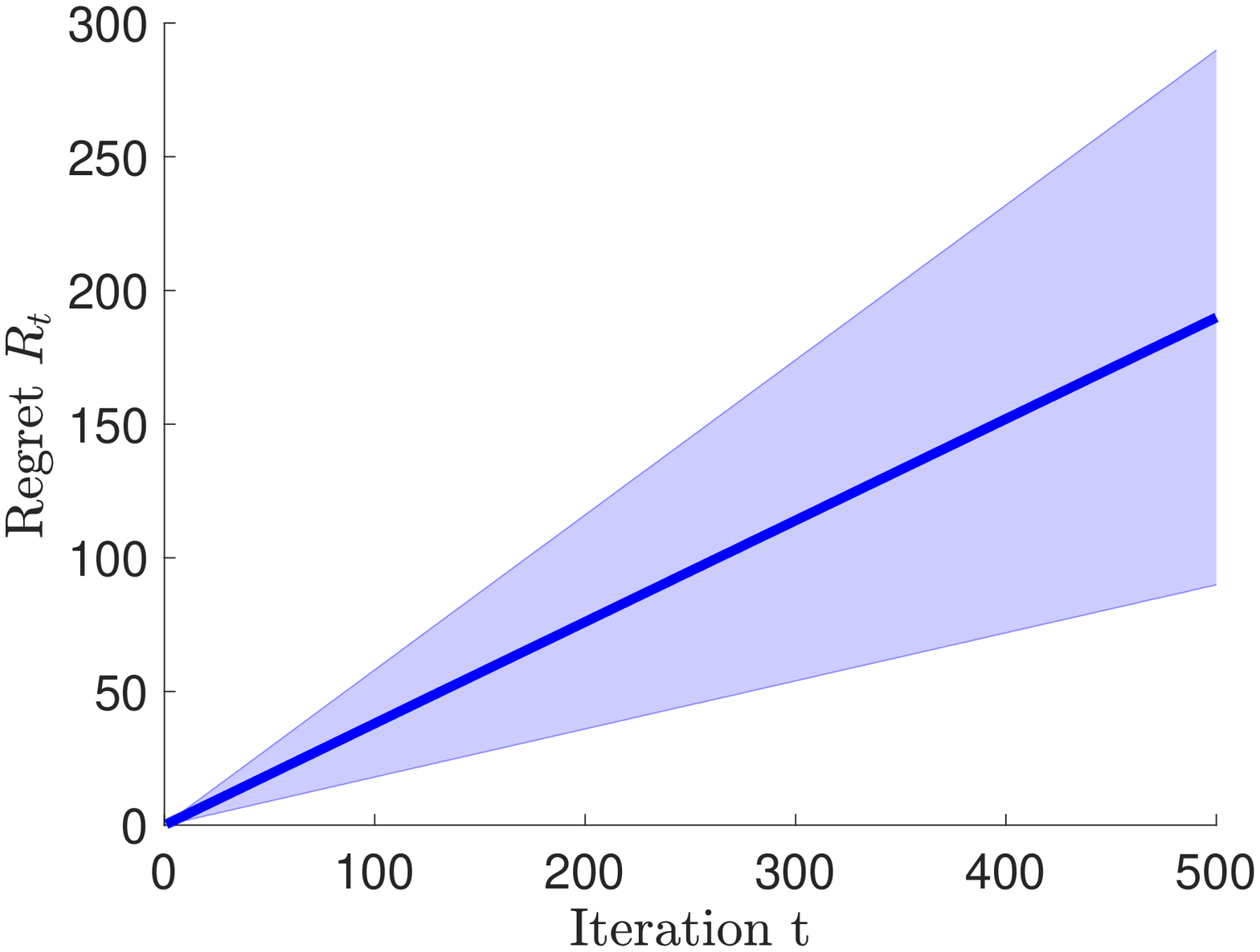}
        \caption{SafeOpt-MC}
          \label{subfig:safeopt1}
     \end{subfigure}
    \hfill 
     \begin{subfigure}[b]{0.32\textwidth}
         \centering
         \includegraphics[width=\textwidth]{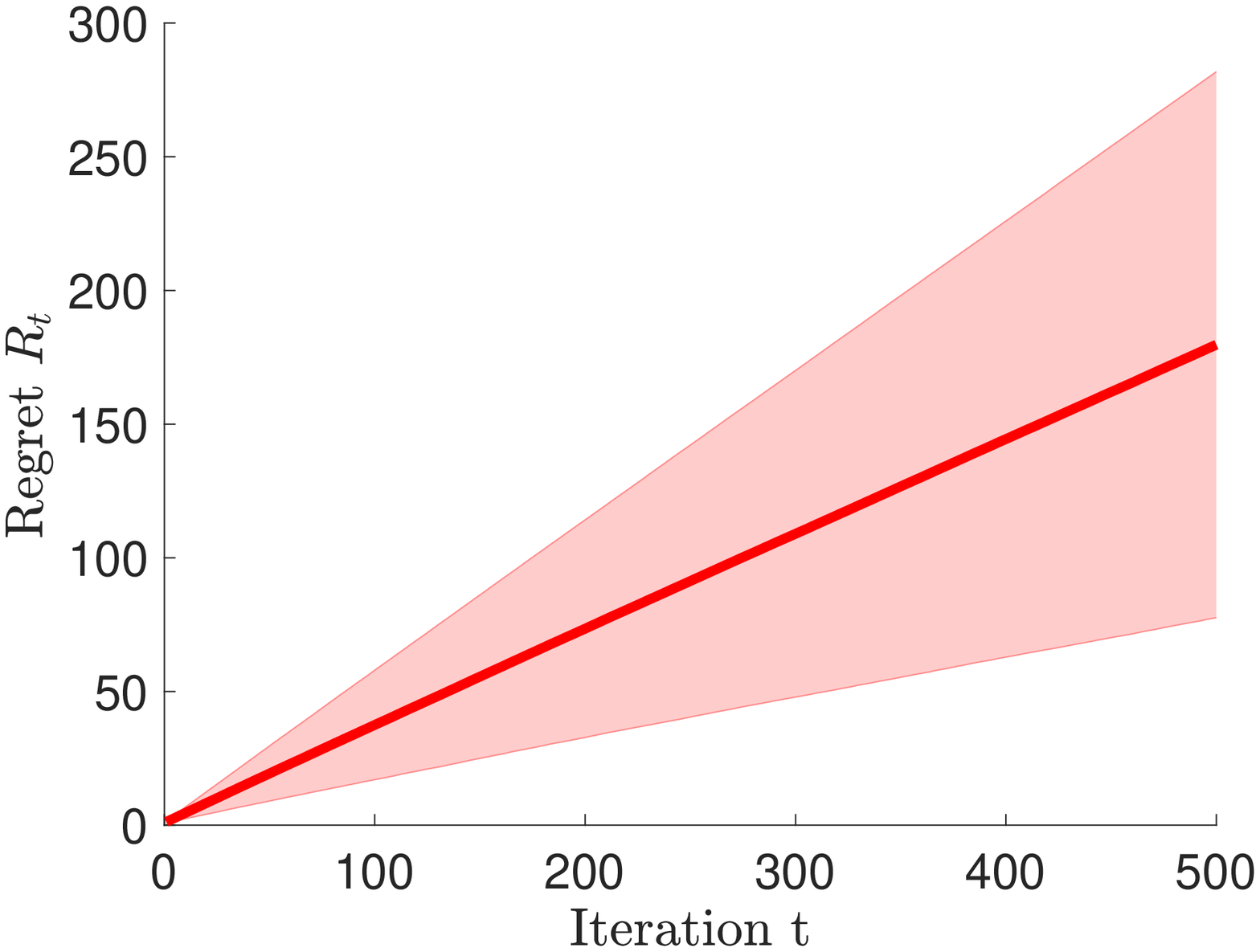}
      \caption{StageOpt}
          \label{subfig:stageopt1}
     \end{subfigure}
     \hfill
     \begin{subfigure}[b]{0.32\textwidth}
         \centering
         \includegraphics[width=\textwidth]{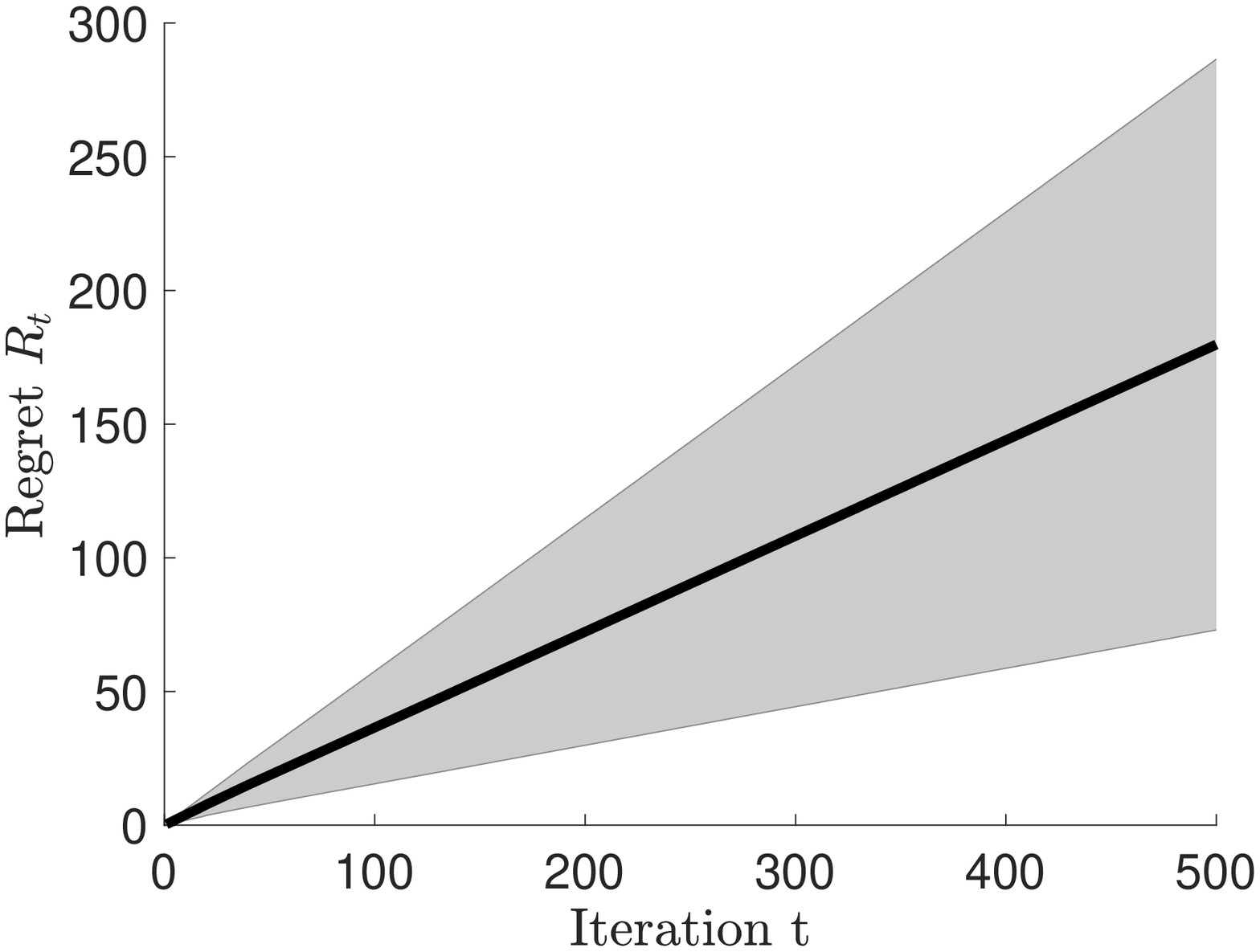}
         \caption{\SGPUCB}
         \label{subfig:sgpucb1}
     \end{subfigure}

\caption{$1\leq|\Dc^w|\leq 10$}
  \label{fig:errorregretcomparison1}
\end{figure}

\begin{figure}
  \centering
  \begin{subfigure}[b]{0.32\textwidth}
         \centering
         \includegraphics[width=\textwidth]{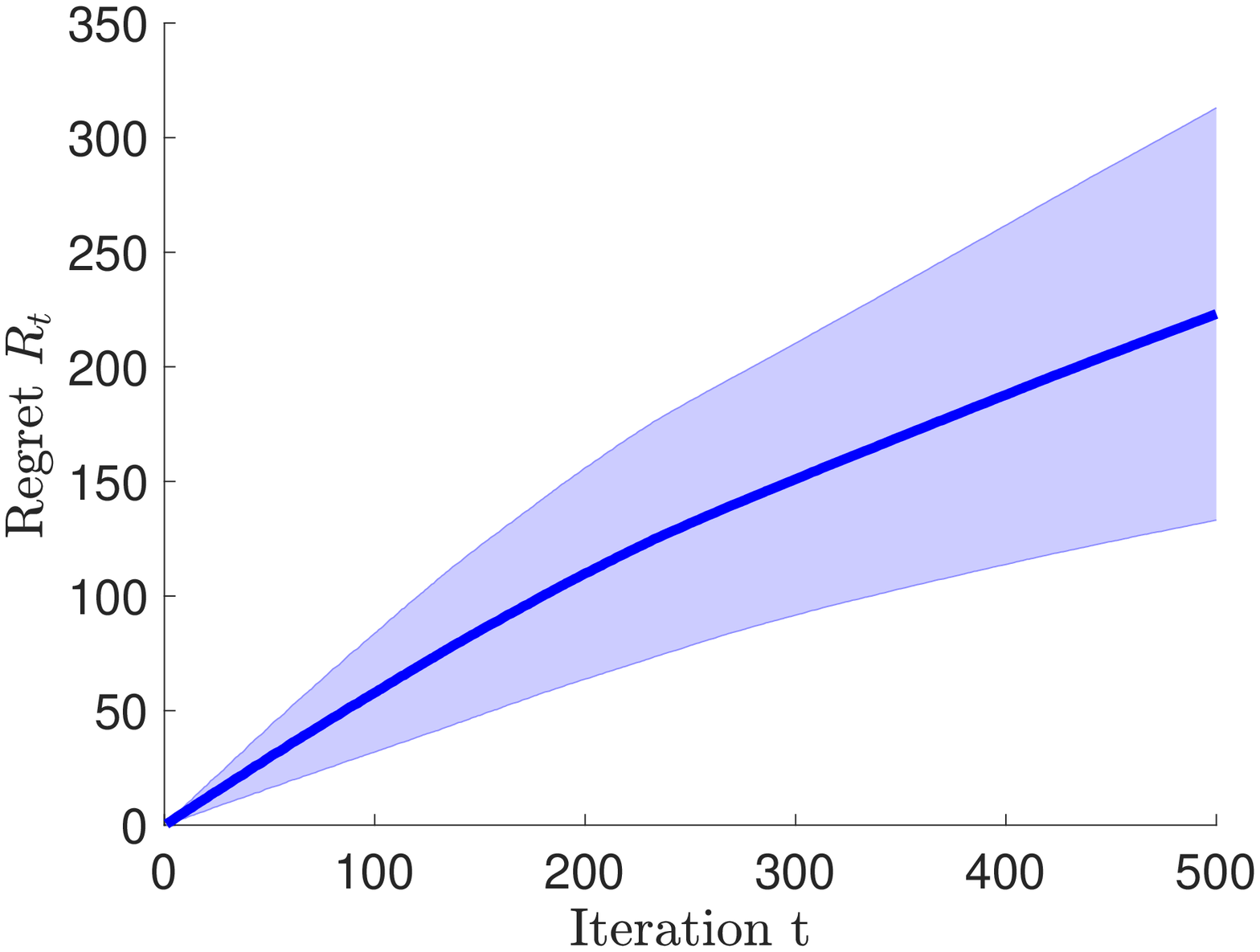}
        \caption{SafeOpt-MC}
          \label{subfig:safeopt2}
     \end{subfigure}
    \hfill 
     \begin{subfigure}[b]{0.32\textwidth}
         \centering
         \includegraphics[width=\textwidth]{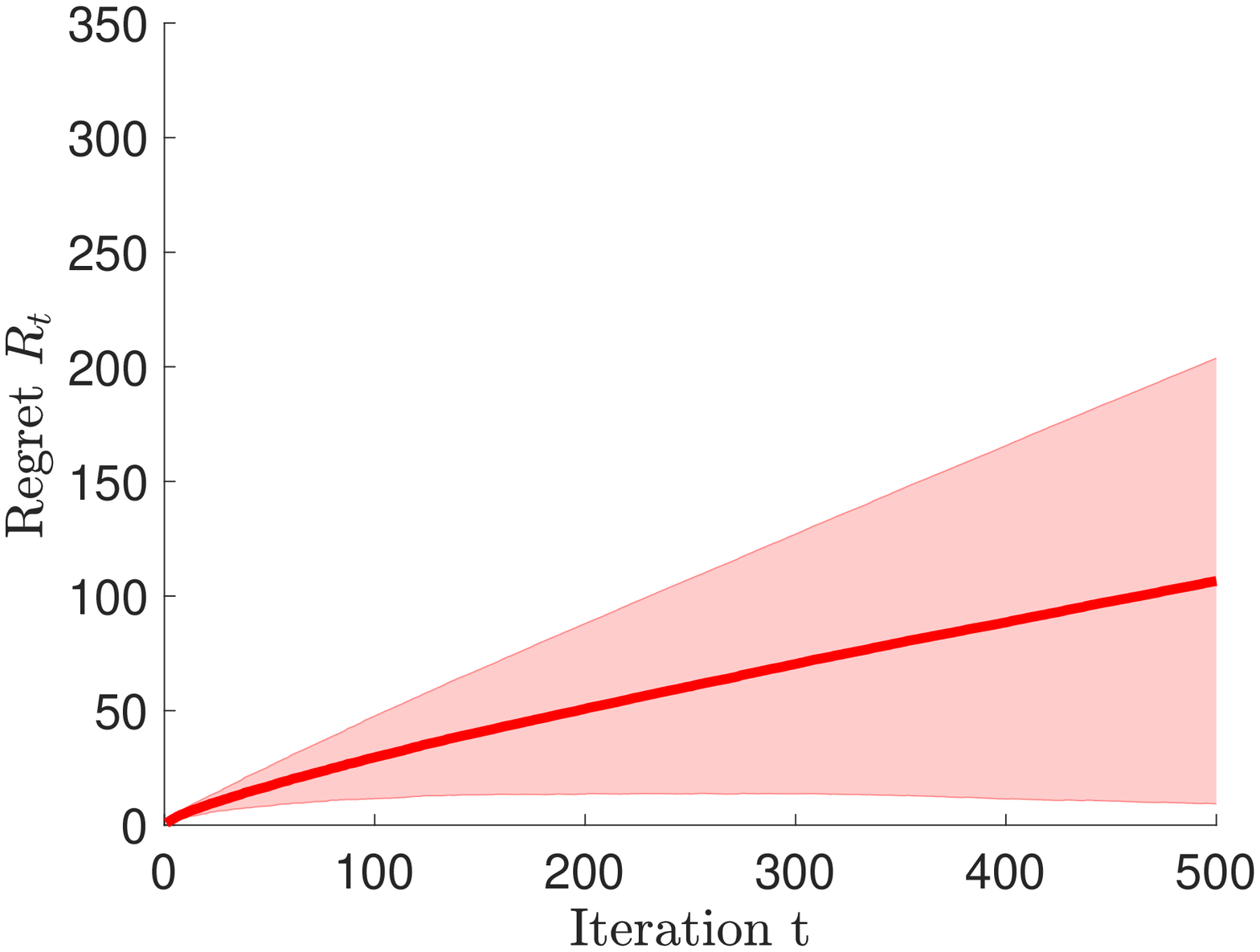}
      \caption{StageOpt}
          \label{subfig:stageopt2}
     \end{subfigure}
     \hfill
     \begin{subfigure}[b]{0.32\textwidth}
         \centering
         \includegraphics[width=\textwidth]{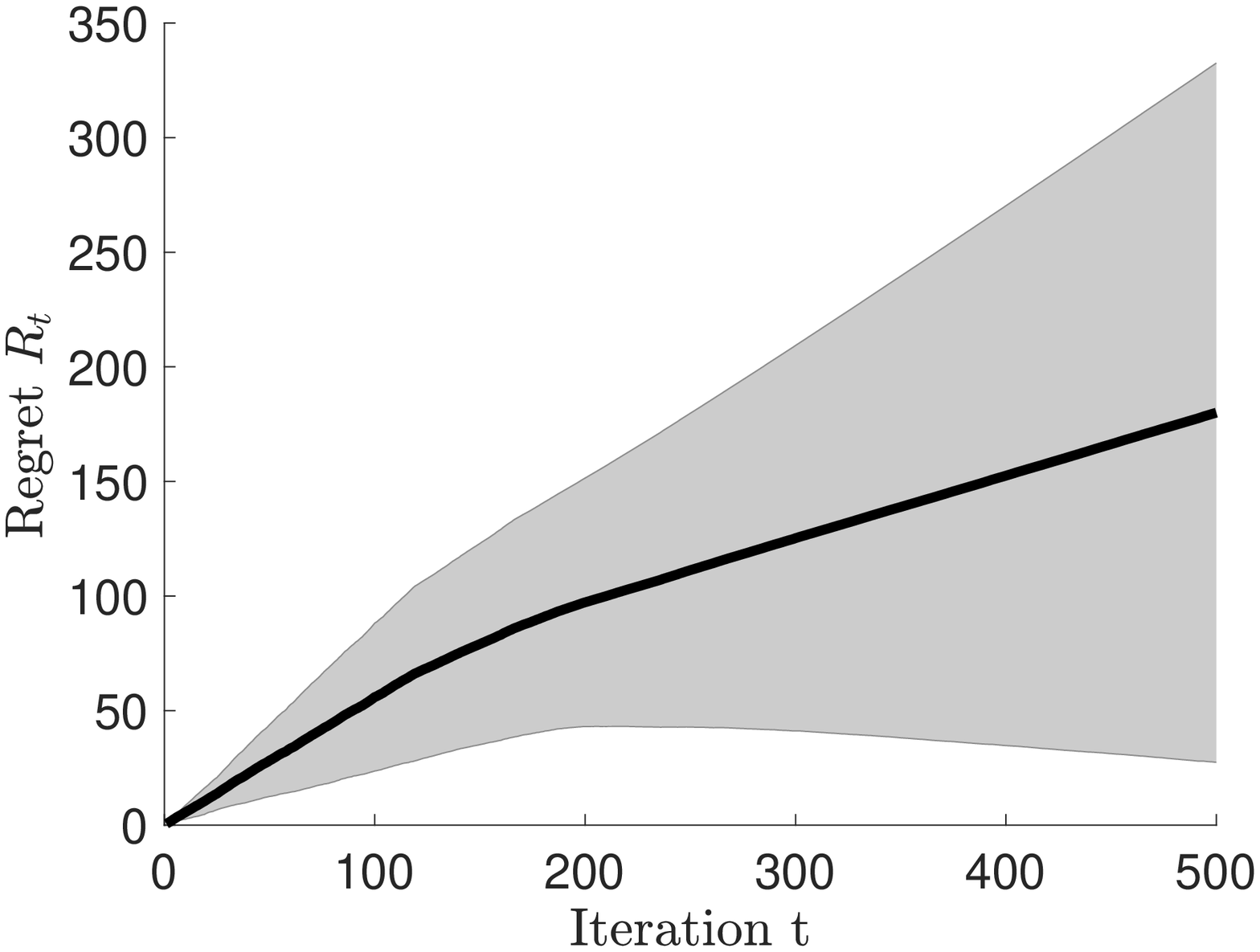}
         \caption{\SGPUCB}
         \label{subfig:sgpucb2}
     \end{subfigure}

\caption{$11\leq|\Dc^w|\leq 20$}
  \label{fig:errorregretcomparison2}
\end{figure}

\begin{figure}
  \centering
  \begin{subfigure}[b]{0.32\textwidth}
         \centering
         \includegraphics[width=\textwidth]{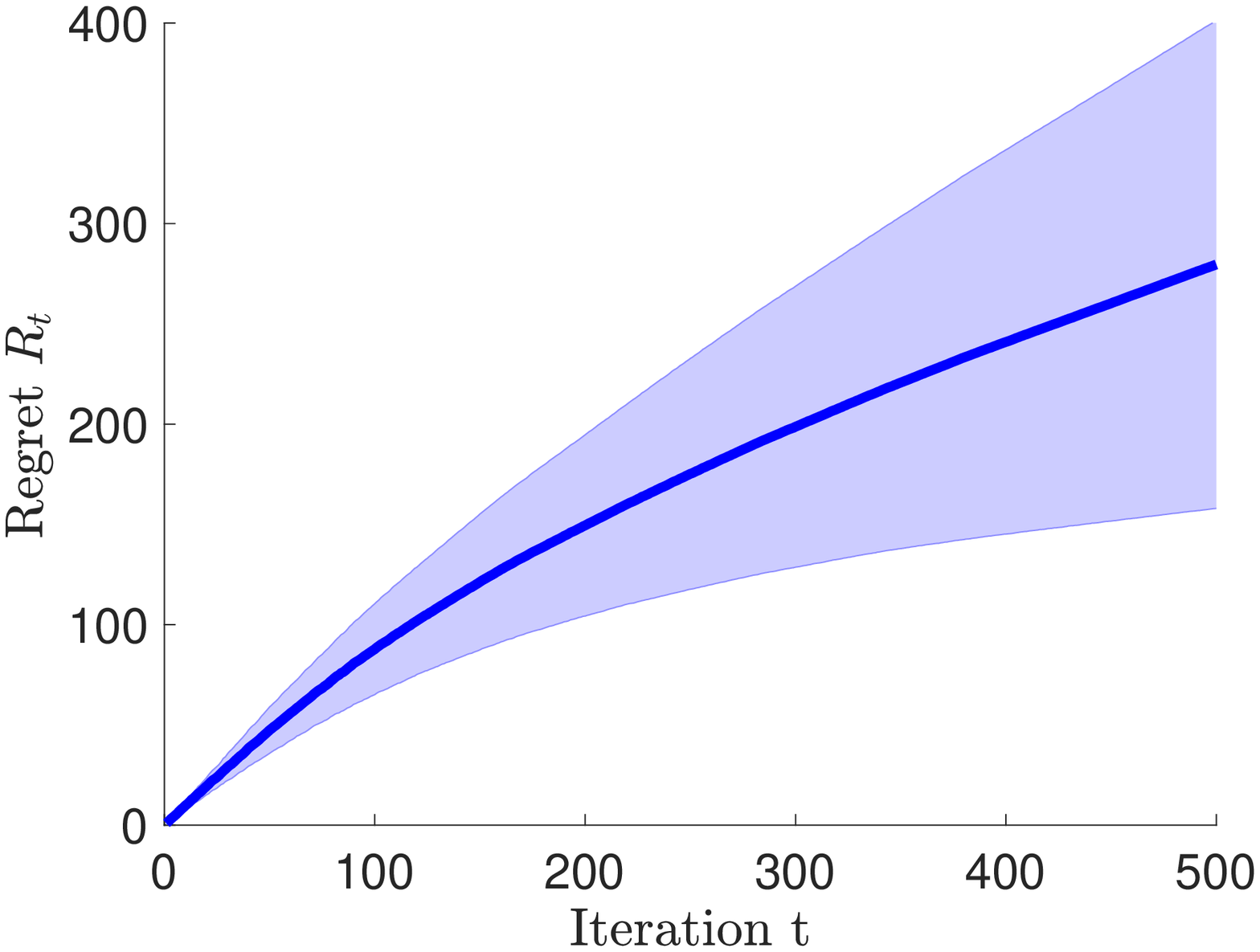}
        \caption{SafeOpt-MC}
          \label{subfig:safeopt3}
     \end{subfigure}
    \hfill 
     \begin{subfigure}[b]{0.32\textwidth}
         \centering
         \includegraphics[width=\textwidth]{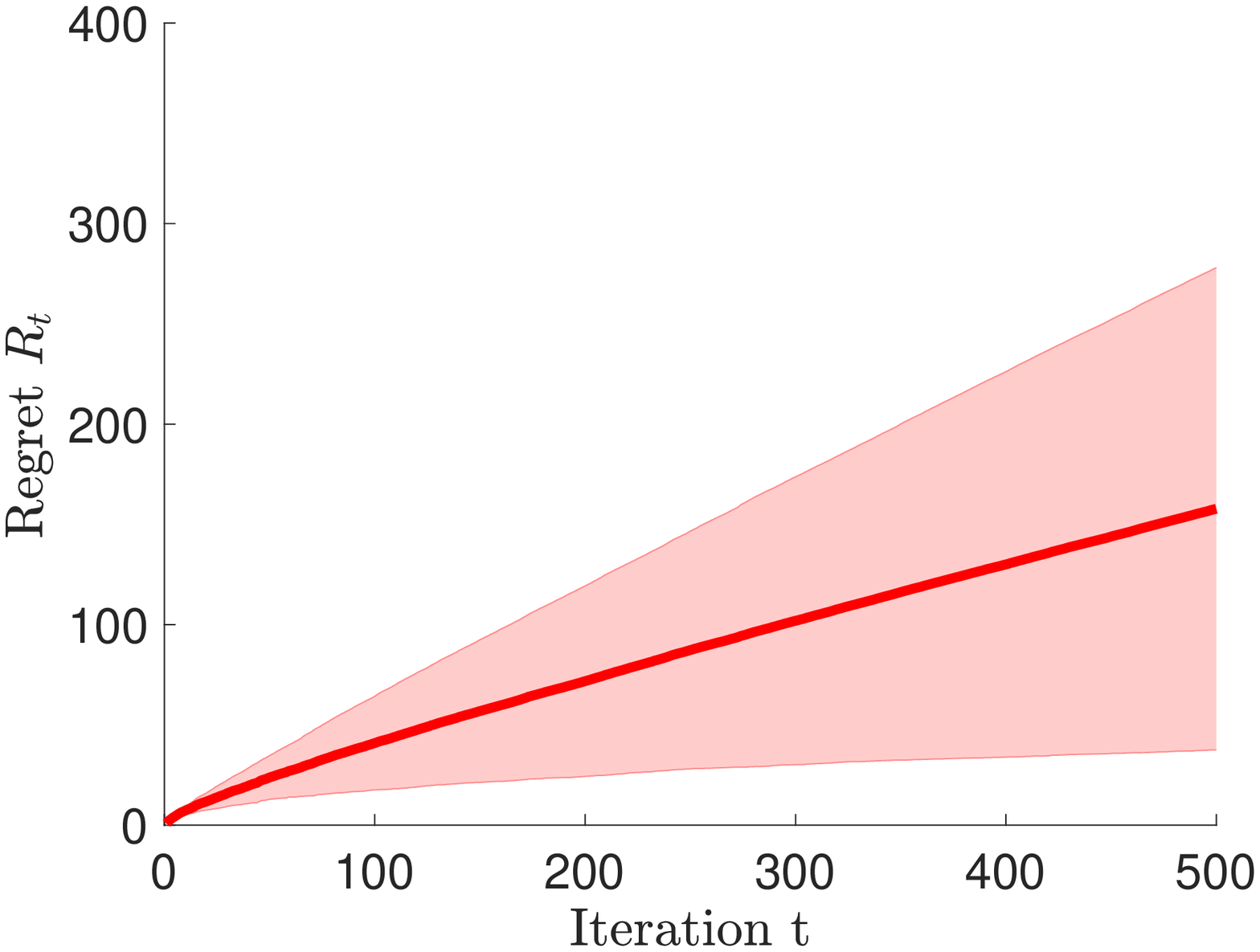}
      \caption{StageOpt}
          \label{subfig:stageopt3}
     \end{subfigure}
     \hfill
     \begin{subfigure}[b]{0.32\textwidth}
         \centering
         \includegraphics[width=\textwidth]{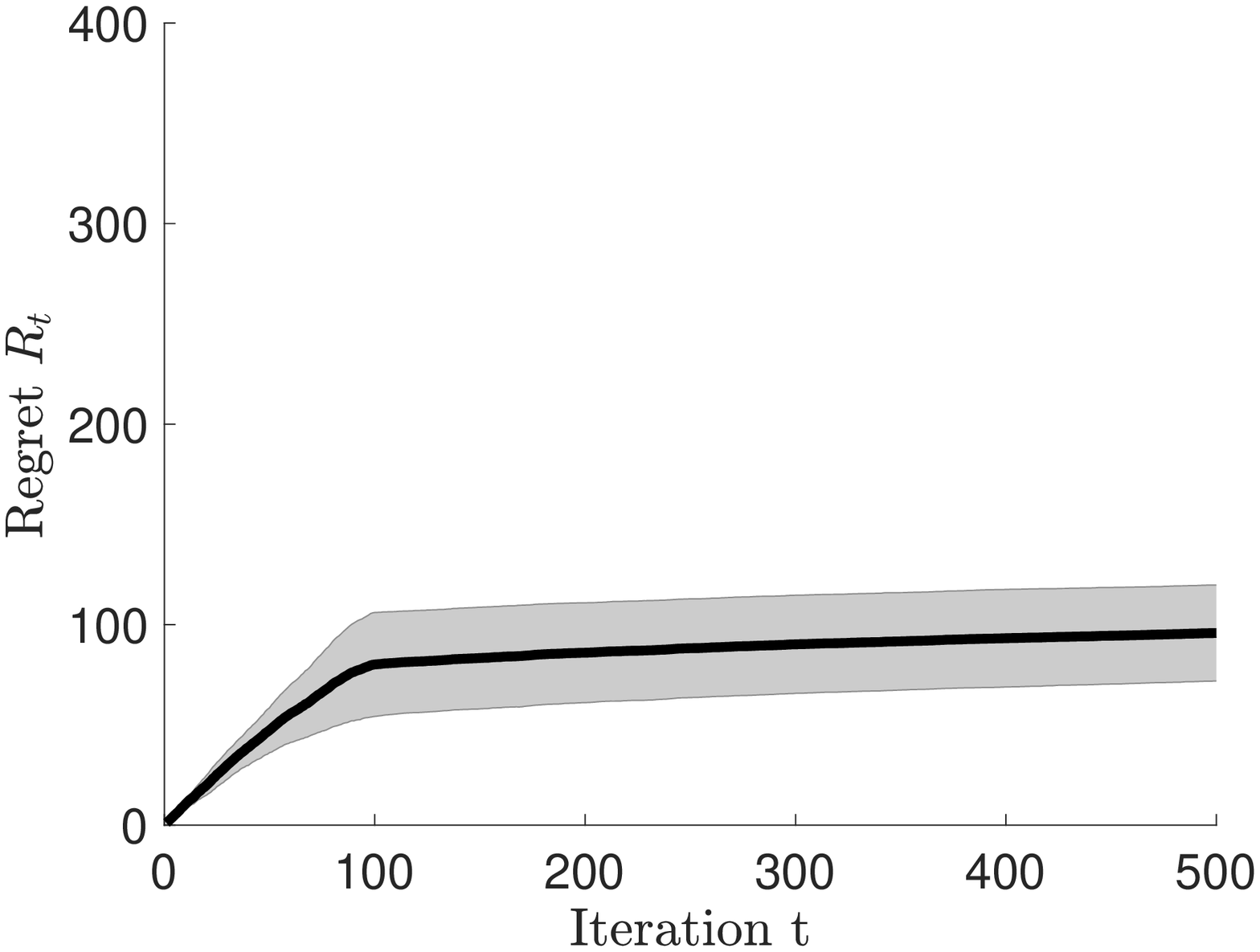}
         \caption{\SGPUCB}
         \label{subfig:sgpucb3}
     \end{subfigure}

\caption{$21\leq|\Dc^w|\leq 25$}
  \label{fig:errorregretcomparison3}
\end{figure}

\section{Numerical study with error bars}\label{sec:appendixsimulation}
In this section, we provide the figures including standard deviation of regret curves presented in Figure \ref{fig:regretcomparison}. Figures \ref{fig:errorregretcomparison1}, \ref{fig:errorregretcomparison2} and \ref{fig:errorregretcomparison3} highlight the standard deviation around the average regret curves depicted in Figure \ref{subfig:badbad}, \ref{subfig:badgood} and \ref{subfig:goodbad}, respectively.



\end{document}